\documentclass[a4paper]{article}

\usepackage{amsbsy,color,multicol}

\usepackage{cite}
\usepackage{caption}
\usepackage{subcaption}
\usepackage{amsmath,amsthm,amsfonts,amssymb, dsfont}
\usepackage{mathrsfs}
\usepackage{graphics}
\usepackage{hyperref}
\usepackage{graphicx}
\usepackage{tikz}

\usepackage{authblk}

\input{epsf}

\newtheorem{theorem}{Theorem}
\newtheorem{lemma}{Lemma}
\newtheorem{corollary}{Corollary}

\title{\bf The Powers of Precision: Structure-Informed Detection in Complex Systems \\ From Customer Churn to Seizure Onset}

\author{
    Augusto Santos\thanks{Augusto Santos (email: augusto.santos@lx.it.pt) is with the Instituto de Telecomunicações, Lisbon, Portugal.} , 
    Teresa Santos\thanks{Teresa Santos (email: teresafsantos21@hotmail.com) is with EY, Lisbon, Portugal.}, 
    Catarina Rodrigues\thanks{Catarina Rodrigues (email: catarinafrodrigues2002@gmail.com) is with Cegid, Braga, Portugal.}, 
    and José M. F. Moura\thanks{José M. F. Moura (email: moura@andrew.cmu.edu) is with the ECE Department at Carnegie Mellon University, Pittsburgh, PA, USA.}
}

\begin{document}

\maketitle
\thispagestyle{empty}
\pagestyle{plain}

\begin{abstract}
Emergent phenomena---onset of epileptic seizures, sudden customer churn, or pandemic outbreaks---often arise from hidden causal interactions in complex systems. We propose a machine learning method for their early detection that addresses a core challenge: unveiling and harnessing a system’s latent causal structure despite the data-generating process being unknown and partially observed. The method learns an optimal feature representation from a one-parameter family of estimators---powers of the empirical covariance or precision matrix---offering a principled way to \emph{tune in} to the underlying structure driving the emergence of critical events. A supervised learning module then classifies the learned representation. We prove structural consistency of the family and demonstrate the empirical soundness of our approach on seizure detection and churn prediction, attaining competitive results in both. Beyond prediction, and toward explainability, we ascertain that the optimal covariance power exhibits evidence of good \emph{identifiability} while capturing structural signatures, thus reconciling predictive performance with interpretable statistical structure.
\end{abstract}

\section{Introduction}
\label{Intro}

Complex systems---ranging from biological and contagion networks to economic markets and online communities---often exhibit emergent macroscopic phenomena that arise from intricate patterns of microscopic interactions~\cite{emergence_phenomena}.~These emergent events, such as epileptic seizures~\cite{Seizure}, the onset of pandemics~\cite{pandemics_complex_systems,Vesp_Complex,augustoasilo, ogura}, market crashes~\cite{Market_Crashes}, or rapid shifts in customer behavior~\cite{Churn_Emergence}, can have profound consequences; detecting them early is thus critical for timely intervention. Yet, this task is hindered by the fact that the \emph{structural laws} governing such systems are latent, domain-specific, and shaped by the statistical and dynamical nature of the data.

Recent advances in network science and statistical learning have underscored the central role of \emph{structure}---the hidden causal and relational architecture of a system---in driving such emergent behavior~\cite{emergent_behavior_ML,paper:CDC}. Therefore, accurate recovery of this structure from observational data can greatly enhance early detection and classification in complex systems. However, such recovery is challenging: classical causal inference and structure-identification methods chiefly rely on strong and often contrived assumptions about the nature of the data-generating process---e.g., Gaussianity, (non)linearity, stationarity, network sparsity, or full observability---risking degraded performance when these assumptions are defied across domains. In particular, these methods tend to lose their technical guarantees of consistency or sample-complexity optimality when deployed across systems whose governing laws detour away from the specifics underlying their design, as in most real-world applications~\cite{SantosA}.

Herein, we introduce a \emph{structure-informed feature representation} built from a one-parameter family of estimators---powers of the empirical covariance matrix. We argue that \emph{structural features}---those entailing information about the causal graph of fundamental dependencies linking state variables---can serve as a \emph{universal signal} for detecting emergent phenomena, provided they adapt to the nature of the underlying data. We observe that the true data-generating law is typically unknown; yet distinct covariance powers yield consistent estimators of a system’s latent structure depending on the generative mechanism---Gaussian graphical models, Structural Equation Models, ferromagnetic Ising models, linear diffusion processes, etc. Selecting the optimal power thus becomes a problem of \emph{structural tuning}---learning, from data alone, which map best preserves the concealed architecture that governs critical events despite the uncertainty about the underlying dynamical law.

The importance of the causal graph structure can be vividly illustrated in two contrasting scenarios. In Networked Dynamical Systems (NDS) such as those modeled by the Kuramoto framework (a widely adopted representation for seizure dynamics~\cite{Kuramoto_Seizure}), whether a set of coupled oscillators synchronizes over time can fundamentally depend upon the structural subtleties of the underlying interaction network---sometimes the presence or absence of a single link seals the fate~\cite{synchronization_sensitivity}. Conversely, within the scope of Structural Causal Models~\cite{pearl_2009}, correlation without structure can mislead: ice cream sales and shark attacks tend to rise and fall together, yet both are driven by a common latent seasonal factor. Without identifying such hidden drivers or the causal structure of influence between variables thereof, predictions and interventions risk going astray---banning ice cream sales will certainly not render anyone safer in the water by averting shark attacks. These examples convey a central claim: only by unravelling the underlying causal mechanisms can we move from \emph{correlation to explanation, from patterns to principles, and from data to knowledge}---the overarching foundation behind our proposed structure-informed feature representation for early detection of emergent phenomena in complex systems.

A further challenge is that structure identification methods often rely on rich temporal or multivariate observations, an assumption that falls short in many real-world applications. For example, in customer churn prediction, only a single sample vector~$\mathbf{x}\in\mathbb{R}^{N}$ per customer is typically available. Thus, we introduce a novel \emph{training–testing adaptation} that preserves the feature design despite such data scarcity. During training, the model uses ensemble-based empirical covariances---each raised to the candidate power---as input features. At test time, the covariance is approximated by the rank-one matrix $\mathbf{x}\mathbf{x}^{\top}$---regularized to be positive definite---and then raised to the learned optimal power, enabling real-time or individual-level classification. 

We evaluate our structure-informed approach with the proposed family over two contrasting case studies: (i) EEG-based early seizure detection, where the method attempts to encode the underlying latent structural connectivity patterns from the multivariate time series, and (ii) customer churn prediction, whereby only a single sample vector is available per client at test time. The advantage of our approach lies in its \emph{universality}: the same computational and mathematical principles apply indelibly across these domains, despite their markedly discerning nature in data modalities.

We note that a practical impediment to supervised learning in both settings lies in their prevailing class imbalance: seizures and churn events are exceedingly rare\footnote{Churn prevalence varies across economic regimes. The dataset considered reflects a stable operational setting, wherein churn is rare yet consequential—a regime that renders prediction critical.}. We breach common practice by eschewing synthetic rebalancing methods like SMOTE~\cite{SMOTE}. Such procedures inexorably bias models toward artificial statistical artifacts, jeopardizing the very structural information our framework seeks to harness. Indeed, drawing synthetic data that is consistent with the latent dependencies would demand precise foreknowledge of the causal graph---the elusive element. Thus, we strive to safeguard the structural signatures at the core of our approach by preserving the natural data, albeit imbalanced, rather than risk engendering spurious associations and thereby forfeiting the structure-informed design.

Beyond predictive performance, we analyze the \emph{identifiability} of the learned features in their natural habitat: the manifold of symmetric positive-definite (SPD) matrices endowed with the Affine–Invariant Riemannian (AIR) metric~\cite{spd_manifold}. We observe that the mean pairwise AIR distance between feature matrices of the same class is consistently shorter than that between features of opposing classes. We further find that the variance of intra-class AIR distances is likewise smaller than its inter-class counterpart. These observations suggest that the learned representation yields well-separated class clusters, indicating that the structural view is not only predictive but also discriminative in a principled sense. This substantiates our main hypothesis that network structure drives emergence, yielding both competitive detection accuracy and interpretable indicators of criticality in complex systems.

The main contributions of this work are now summarized:
\begin{itemize}
\item \textbf{Structural Consistency under Partial Observability.}
We prove that, for graph-based Matérn random fields, suitable powers of the covariance matrix are \emph{structurally consistent estimators} of the latent interaction graph even when only a subset of nodes is observed. 
\item \textbf{Feature Design.} As a result, we introduce a \emph{structure-informed feature representation} based on a one-parameter family of estimators—powers of the empirical covariance matrix. These are fed as input features to supervised learning models for classification.
\item \textbf{Training Adaptation.} We propose a \emph{training–testing adaptation} for \emph{small-data} regimes, enabling one-shot, individual-level predictions whilst preserving structural informativeness learned during training.
\item \textbf{Feature Identifiability.} We conduct a simple empirical analysis on the manifold of symmetric positive-definite matrices (SPD) endowed with the AIR-metric, providing indicative evidence of non-trivial discriminative power of the proposed structure-informed features.
\item \textbf{Structural Signatures.} We link predictive performance to interpretable statistical structure by revealing that the learned optimal powers capture meaningful structural signatures underlying emergent phenomena.
\end{itemize}

\section{Related work}

A unifying theme across modern neuroscience, network science, and machine learning is that 
\emph{structure—rather than correlation—governs emergent behavior}. Complex systems are mediated by latent interaction graphs: 
inter-regional communication along structural connectome pathways in the brain 
\cite{science_brain_communi,brain_communic}, 
homophily or influence networks in customer ecosystems \cite{homofilia}, 
regulatory and interaction architectures in physiological and biological systems \cite{GRN_net}, 
and social learning in networked populations \cite{Javidi,tomo_socialicassp,social_learning_book}. 
In these settings, critical events such as seizure onset, pandemic outbreaks, disease progression, 
customer churn, or collective belief formation arise as macroscopic manifestations of hidden network interactions. 
Yet, in most applications, neither the interaction graph nor the data-generating law is known, fundamentally constraining feature design and inference.

\subsection{Representation Learning for Domain Applications}

Within this broader landscape, customer churn prediction has traditionally focused on improving predictive performance through increasingly sophisticated machine learning pipelines. Classical approaches combine standard classifiers such as Naïve Bayes, decision trees, and ensemble methods with optimization or evolutionary strategies to address automatic representation learning challenges \cite{amin2023adaptive_churn_nb_ec,wagh2024churn_telecom_ml,prabadevi2023customer_churn_ml}. More recently, churn prediction has evolved from traditional statistical models such as logistic regression and survival analysis \cite{verbeke2012new2,churn_pred_mutual} to high-capacity deep learning architectures \cite{keramati2014improved,manzoor2024review_churn}, including hybrid frameworks \cite{ouf2024hybrid_churn_telecom}, deep temporal models \cite{yang2025cnn_bilstm_attention_churn,liu2024hybrid_nn_churn}, and integrated segmentation--prediction pipelines \cite{wu2021integrated_churn_segmentation}. Recent work has also underscored the role of graph-based feature engineering, noting that certain customer interactions foster implicit network effects conducive to churn dynamics \cite{lee2025tempodegraphnet,nimma2025integrated_sentiment_gnn}. While these methods often achieve strong empirical performance, the resulting graph representations remain largely implicit and task-specific, offering limited interpretability. Across these approaches, the structural minutiae of relational dependencies—e.g., service bundles, billing behavior, or contract characteristics—are typically ignored or incorporated implicitly through task-driven representation learning, without guarantees of stability or consistency under partial observability or distributional shift.

On the other hand, in neuroscience, a growing body of literature demonstrates that brain function and learning are shaped by 
the topology of neural interaction graphs.
Foundational surveys \cite{Bassett_Network_Neuroscience} and recent empirical studies 
show that network organization governs signal flow, representational geometry, and learning efficiency 
\cite{kahn2025network_structure_representations,general_network_emergence_brain,general_brain_network}. In health sciences, in general, incorporating patient similarity networks and physiological interaction graphs has improved early disease detection \cite{rajkomar2018scalable,barabasi2011interactome,stam2014modern}.
These works highlight the importance of structural priors, but do not address structural recovery under diffusion and partial observability--a challenging problem that remains an active area of research~\cite{tomo_isit,PAMI_Partial,Geigeretal15,pmlr-v286-negro25a}.
These findings support the view that macroscopic neural activity patterns emerge from latent structural constraints,
motivating methods that aim to exploit or recover this structure from observed data.

Along this line, epileptic seizures are increasingly construed as emergent phenomena rather than local signal abnormalities.
Early work relied primarily on spectral or univariate EEG features \cite{mormann2007seizure}, spanning classical signal-processing pipelines and deep temporal architectures. Representative approaches include hybrid methods combining handcrafted time--frequency decompositions with conventional classifiers—such as Fourier- or empirical mode--based features coupled with support vector machines or gradient boosting \cite{saidi2021cnn_svm_seizure,amer2023pft_eeg,wu2020eemd_xgboost_seizure}—as well as purely data-driven convolutional and recurrent neural networks \cite{alharthi2022epileptic_eeg_detection,abdelhameed2021deep_seizure_children,wang2021one_dimensional_cnn_seizure,zhang2022bidirectional_gru_seizure,liu2024hybrid_cnn_lstm_seizure}. These methods constitute strong empirical baselines but largely treat EEG recordings as multivariate time series, with inter-channel dependencies either waved or absorbed as surrogate in the learning architecture.


Motivated by the evidence that seizure dynamics are shaped by network interactions, subsequent work advocated explicit representation for functional connectivity and network-level dynamics. This perspective led to graph-based learning approaches, including causality-informed pipelines that estimate directed functional networks prior to learning \cite{wang2020dtf_cnn_seizure} and graph neural networks that operate on learned or predefined EEG graphs to capture temporal and inter-channel dependencies \cite{truong2018convolutional,liu2020deep,covert2019temporal_gcn_seizure,yang2022eeg_gnn_seizure,razi2022efficient_gcn_seizure,quadri2024stacked_cnn_bilstm_seizure}. More recent models further incorporate attention mechanisms and dynamic connectivity estimation to capture spatiotemporal dependencies in EEG \cite{Automated_Seizure_Detection,Seizure_ICLR,Seizure_Volkan_ICML}. 

While modeling inter-channel interactions can improve predictive performance, existing approaches typically learn connectivity implicitly within high-capacity models, yielding limited interpretability and no guarantees of consistent recovery of the latent interaction structure under partial observability. In contrast, our work shifts the focus from architectural complexity to principled structure-informed feature design, achieving competitive performance with standard classifiers and without deploying artificial data rebalancing.

\subsection{Matérn Random Fields over Graphs}

A central difficulty in structural recovery--and, by extension, in structure-informed feature design--is that observed data typically reflect the outcome of a diffusion process already integrated over time. More concretely, due to limited temporal observability, each measurement often aggregates 
multiple unobserved propagation steps—effectively encoding multi-hop interactions and obscuring local network structure.\footnote{
For example, sampling a Markov process at coarse time intervals yields observations that conflate many intermediate transitions.}
Matérn random fields provide a minimal and principled mathematical framework for this setting.
In both continuous and graph-structured domains, a Matérn field is governed by a fractional diffusion operator of the form
$(\kappa^2 D + L)^{\alpha/2}$ \cite{borovitskiy2021graphmatern}, 
where the Laplacian $L$ encodes network structure and the exponent $\alpha\in\mathbb{R}$ controls the accumulation of multi-hop interactions.
This renders Matérn models particularly pertinent to scenarios of partial temporal observability,
where local interactions are smoothed or mixed by diffusion before being observed.

In the brain, the precise generative law governing large-scale activity remains unknown and is likely heterogeneous across spatial scales, brain states, and individuals. Nevertheless, Matérn random fields provide a principled mathematical framework that embodies a minimal ontological commitment: diffusion-driven propagation, stochastic forcing, and multi-scale or long-range spatial interactions due to partial temporal observability. Indeed, Matérn models are now widely used in spatial statistics \cite{lindgren2011spde} and have gained substantial traction in neuroimaging,
where they capture smooth yet heterogeneous spatial organization of neural activity 
\cite{mejia2020bayesian,spencer2022spatial,siden2021bayesian,lindgren2024tenyears}.
Extensions to graphs replace the Laplace--Beltrami operator with a graph Laplacian, yielding Matérn Gaussian processes over graphs 
\cite{borovitskiy2021graphmatern}.
However, existing work assumes the interaction graph is known \emph{a priori};
whether latent network structure can be consistently recovered from spatially partially observed Matérn random fields has remained unexplored. In this direction, empirical observations~\cite{brainaugusto} suggest that certain transforms of the covariance matrix may encode nontrivial structural information of brain activity—a phenomenon that we formalize and prove for Matérn models under partial observability.

We evaluate the proposed structure-informed representations on seizure detection and churn prediction—two structurally distinct systems in nature, arising from biological and socio-economic processes, respectively. These benchmarks are deliberately chosen to emphasize that our approach targets universal structural mechanisms underlying emergent phenomena in complex systems, rather than domain-specific generative models or signal characteristics.

\section{Problem formulation}

\label{sec:fractional_fields}

We consider a graph-based Matérn random field $\mathbf{y} = \left(y_1,\ldots,y_N\right)\in\mathbb{R}^N$ defined over the nodes of an underlying interaction graph, where $y_i$ is the field (or state) associated with node $i$. Specifically, $\mathbf{y}$ is generated as the solution to the stochastic equation (degree weighted generalization of the Matérn model in~\cite{borovitskiy2021graphmatern})
\begin{equation}\label{eq:GraphMatern}
\left(\kappa^2 D + L\right)^{\alpha/2}\mathbf{y} \;=\; \mathbf{x},
\end{equation}
where $L = D-A \in\mathbb{S}^{N\times N}$ denotes the graph Laplacian encoding the latent interaction structure, $A$ is the latent ground-truth weighted interaction matrix, $D={\sf diag}(A\mathbf{1})$ is the corresponding weighted degree matrix conveying the weighted degrees of the nodes across its main diagonal, $\kappa^2>0$ is a scale parameter, and $\alpha\in\mathbb{R}$ controls the smoothness of the field, or equivalently the accumulation of multi-hop interactions induced by diffusion on the graph. The driving noise $\mathbf{x}$ is a zero-mean random vector with homogeneous diagonal covariance $\Sigma = \mathbb{E}\!\left[\mathbf{x}\mathbf{x}^{\top}\right] = \sigma^2 I$.

In classical Matérn theory, $\alpha>0$ ensures a valid generative Gaussian random field. 
In this work, while the generative interpretation is anchored in this regime, 
we allow $\alpha\in\mathbb{R}$ when constructing covariance and precision powers as analytical estimators, 
motivated by both empirical evidence and our theoretical results showing that such extensions preserve structural information under partial observability.

The fractional-field model~\eqref{eq:GraphMatern} subsumes several widely used
statistical models on graphs as special or limiting cases, depending upon the parameters $\kappa$, $D$, $\alpha$ regimes, thereby providing
a unified parameterized framework. Namely, it entails Structural Equation Models (SEM), Gaussian Graphical Models (GGM), limiting distribution of diffusive limits, to cite a few examples~\cite{borovitskiy2021graphmatern}. Together, these correspondences show that the graph-based Matérn model
defines a unified family of graphical models interpolating between SEMs,
Gaussian graphical models and equilibrium diffusions with the parameter $\alpha$ critically governing the modeling regime by controlling how
local interactions accumulate across observations.

\paragraph{Powers of the precision matrix and structural recovery.}
In general, the interaction matrix $A$—and hence the Laplacian $L = D - A$—is unknown and must be inferred from observations of the random field $\mathbf{y}$.
From the defining equation of the graph-based Matérn random field in~\eqref{eq:GraphMatern}, we may write
\begin{equation}
\mathbf{y}
= \left(\kappa^2 D + L\right)^{-\alpha/2}\mathbf{x},
\end{equation}
where $\mathbf{x}\sim \mathcal{N}(0,\sigma^2 I)$. Therefore, the population covariance matrix $C \;=\; \mathbb{E}\!\left[\mathbf{y}\mathbf{y}^{\top}\right]$ is given by
\begin{equation}\label{eq:matern_cov}
\begin{aligned}
C
&= \left(\kappa^2 D + L\right)^{-\alpha/2}
\,\mathbb{E}\!\left[\mathbf{x}\mathbf{x}^{\top}\right]\,
\left(\kappa^2 D + L\right)^{-\alpha/2} \\
&= \sigma^2 \left(\kappa^2 D + L\right)^{-\alpha}.
\end{aligned}
\end{equation}

Equation~\eqref{eq:matern_cov} shows that the covariance of the observed field is a
fractional inverse power of the latent interaction operator
$\kappa^2 D + L$.
Equivalently, the corresponding precision matrix yields $C^{-1} = \sigma^{-2}\left(\kappa^2 D + L\right)^{\alpha}.$
More generally, taking arbitrary real powers of the covariance yields
\begin{equation}\label{eq:cov_powers}
C^{-\beta}
= \sigma^{-2\beta}\left(\kappa^2 D + L\right)^{\alpha\beta},
\qquad \beta\in\mathbb{R}.
\end{equation}

Equation~\eqref{eq:cov_powers} highlights that suitable powers of the covariance
(or precision) matrix act as analytic transforms of the latent interaction operator.
In particular, when $\beta = 1/\alpha$,
\begin{equation}\label{eq:structural_recovery}
C^{-1/\alpha}
= \sigma^{-2/\alpha}\left(\kappa^2 D + L\right),
\end{equation}
and we can recover the underlying ground-truth interaction network structure encoded in the off-diagonal entries of $L$.

This observation is central to our approach: although the interaction graph is
not directly observed, its structure is embedded in appropriate powers of the
population covariance. Since the covariance $C$ can be consistently estimated
from samples of $\mathbf{y}$, Equation~\eqref{eq:structural_recovery} shows that
covariance powers provide \emph{structurally consistent estimators} of the
underlying interaction operator under the Matérn generative model.

\section{Structural Consistency under Partial Observability}

In practice, the full state vector $\mathbf{y}\in\mathbb{R}^N$ is rarely observed.
Moreover, when deployed in learning pipelines, feature selection further
reduces the set of available observables.
Therefore, actual measurements are available only on a subset of nodes
$\mathcal{S}\subset\{1,\dots,N\}$.
Let $\mathbf{y}_{\mathcal{S}}$ denote the subvector of $\mathbf{y}$ indexed by $\mathcal{S}$. It is important to understand when structural information is still conveyed in the observables (despite the influence by confounders) or whether this information is fundamentally compromised. In this regard, we ask whether there exists a power of the covariance computed
solely from the observable nodes that still conveys structural information about
the latent interaction graph. That conforms to our main technical contribution stated in Theorem~\ref{thm:partial_struct_consistency}.
  
Let $\Pi_{\mathcal{S}}$ be the corresponding coordinate projection matrix.
The covariance of the observed field is then
\begin{equation}
C_{\mathcal{S}}
:= \mathbb{E}\!\left[\mathbf{y}_{\mathcal{S}}\mathbf{y}_{\mathcal{S}}^{\top}\right]
= \Pi_{\mathcal{S}}\, C \,\Pi_{\mathcal{S}}^{\top},
\end{equation}
that is, the principal submatrix of the population covariance $C$ supported on
$\mathcal{S}$.


Under full observability, structural recovery is achieved via suitable powers of the
covariance (or precision) matrix.
Under partial observability, however, one can only estimate powers of the observed
covariance $C_{\mathcal{S}}$ from samples of $\mathbf{y}_{\mathcal{S}}$.
A fundamental difficulty arises from the fact that projection and matrix
exponentiation do \emph{not} commute:
\begin{equation}\label{eq:noncomm}
\left(C_{\mathcal{S}}\right)^{-\beta}
\;\neq\;
\left[C^{-\beta}\right]_{\mathcal{S}},
\qquad \beta\in\mathbb{R},
\end{equation}
where the latter term~$\left[C^{-\beta}\right]_{\mathcal{S}}$ would be the desirable estimator to compute since it contains the structural information as discussed around equation~\eqref{eq:structural_recovery}. However, to compute this term, we would need information about all nodes.

To quantify this difference, we define the commutation error
\begin{equation}\label{eq:delta_alpha}
\Delta_{-\beta}
:= \left([C]_{\mathcal{S}}\right)^{-\beta}
- \left[C^{-\beta}\right]_{\mathcal{S}}.
\end{equation}
When $\Delta_{-\beta}=0$, powers of the partially observed covariance~$\left(C_{\mathcal{S}}\right)^{-\beta}$ inherit the same
structural information as the corresponding principal submatrix $\left[C^{-\beta}\right]_{\mathcal{S}}$. This holds, e.g., whenever
$C_{\mathcal{S}\mathcal{S}'}=0$, i.e., the observed and latent variables are
uncorrelated and the covariance is block diagonal--clearly not the general case.

The central question addressed in this section is therefore:
\emph{under what conditions does $\left(C_{\mathcal{S}}\right)^{-\beta}$ remain a
structurally consistent estimator of the latent interaction graph, despite the
presence of unobserved nodes?}

%

We answer this question for graph-based Matérn random fields by showing that,
under mild and interpretable conditions on the interaction operator and the
noise, suitable covariance powers preserve structural consistency even under
partial observability.


The following result characterizes when the commutation error $\Delta_{-\beta}$
is sufficiently small to preserve the sparsity pattern of the latent interaction
operator.

\begin{theorem}[Structural consistency under partial observability]
\label{thm:partial_struct_consistency}
Let $\mathbf{y}\in\mathbb{R}^N$ follow the graph-based Matérn model
\[
\mathbf{y} = \mathcal{L}^{-\alpha/2}\mathbf{x},
\qquad
\mathcal{L} := \kappa^2 D + L,
\qquad
\mathbf{x}\sim\mathcal{N}(0,\sigma^2 I),
\]
with $\alpha\neq 0$, where $L = D-A$, $A=A^\top\ge 0$, and 
\begin{equation}
\rho(I-\mathcal{L})<1,
\end{equation}
which grants invertibility of $\mathcal{L}$. If $\|A_{\mathcal S\mathcal S'}\| \le g(\mathcal{L},\alpha)$,
where $g(\mathcal{L},\alpha)>0$ is defined in Appendix~\ref{app:frac-consistency} and depends upon the regime ($\alpha>1$, or $\alpha>0$ or general $\alpha\neq 0$), then $(C_{\mathcal S})^{-1/\alpha}$ is structurally consistent: all off-diagonal entries of $(C_{\mathcal S})^{-1/\alpha}$ associated with disconnected pairs lie below the smallest entry across connected pairs.
\end{theorem}

The condition $\|A_{\mathcal S\mathcal S'}\| \le g_\alpha(A)$ formalizes a mild
latent-confounding regime, ensuring that unobserved nodes do not exert a
dominant \emph{coherent} influence on the observed subsystem. The proof to Theorem~\ref{thm:partial_struct_consistency} is provided in Appendix~\ref{app:frac-consistency}.

\textbf{Remark.} The proof to Theorem~\ref{thm:partial_struct_consistency} presented in Appendix~B is novel and may be of independent interest. We rely on the Dunford-Taylor integral characterization of the powers $\left(C_{\mathcal S}\right)^{-\beta}$ and $\left[C^{-\beta}\right]_{\mathcal S}$ to represent them as a transform of their fixed inverse powers. With this, we could resort to Schur complement to design proper bounds on the gap error $\Delta_{-\beta}$ and establish the conditions wherein $\left(C_{\mathcal S}\right)^{-\beta}$ is structurally consistent--or equivalently, the matrix projection and exponentiation commute in a structural sense.  

\section{Methodology}
\label{sec:method}

We leverage the fact that latent interaction structure is encoded in analytic
powers of the covariance matrix, and that this structural blueprint is stable against partial observability (Theorem~\ref{thm:partial_struct_consistency}).
Because the data-generating mechanism is unknown in practice, we treat covariance
power transforms as a \emph{model-agnostic} one-parameter family of structure-informed
representations and learn the exponent directly from data using validation
performance.

Remark that from a feature design perspective, this approach naturally fulfills a minimal set of competing 
desiderata: structural interpretability, robustness to partial observability, 
and adaptability to unknown generative mechanisms.

\subsection{Supervised learning pipeline}
\label{subsec:pipeline}

\begin{figure}[t]
	\centering
	\includegraphics[scale=0.55]{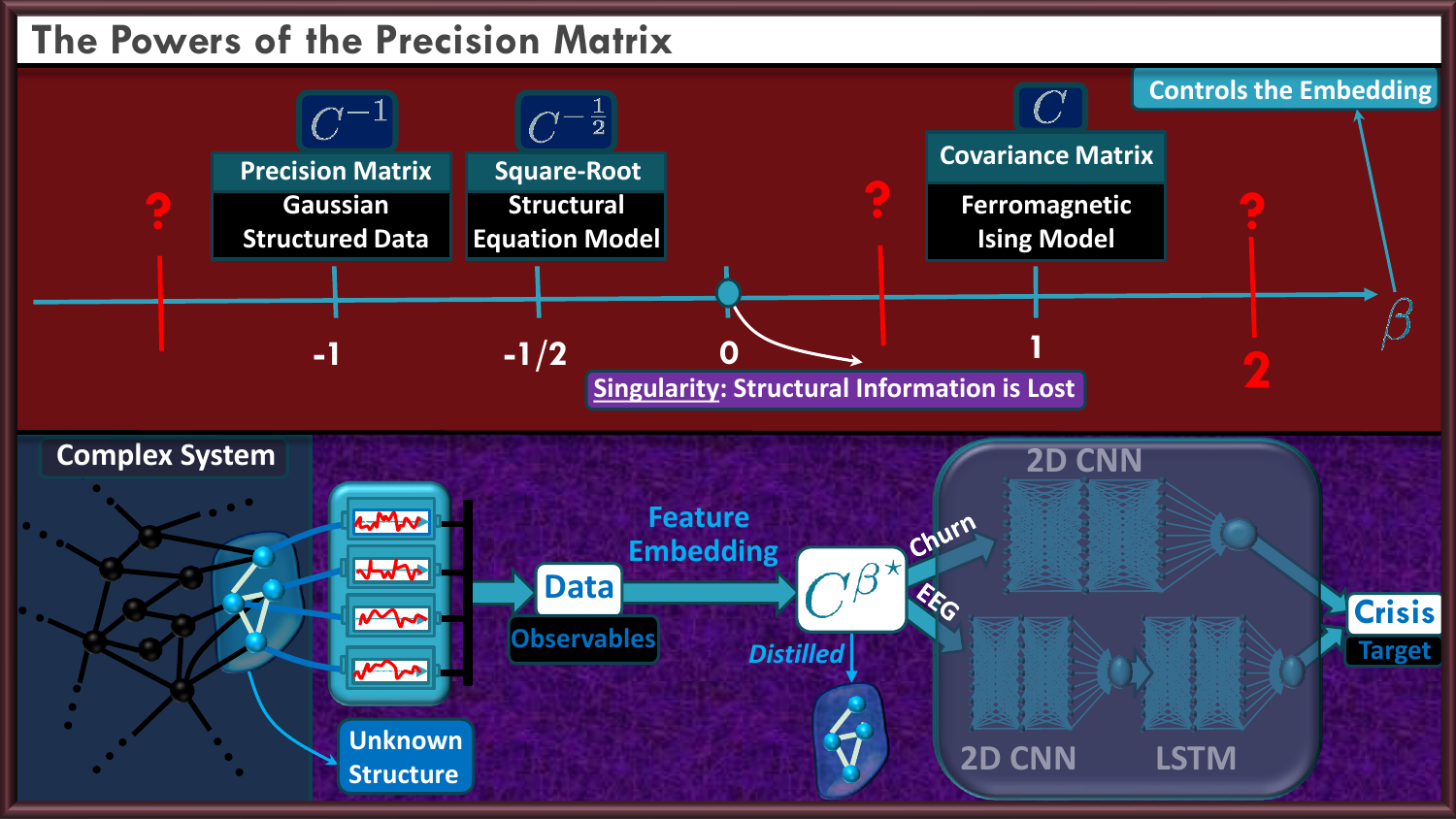}
	\caption{Proposed approach. \textbf{Top:} theoretical regimes in which specific
	covariance powers are structurally consistent. \textbf{Bottom:} pipeline:
	data $\mapsto$ covariance $\mapsto$ power transform $C^{\beta}$ $\mapsto$
	supervised classifier; $\beta^\star$ is selected on a train--validation split.}
	\label{fig:PowersPrecision1}
\end{figure}

Fig.~\ref{fig:PowersPrecision1} summarizes the method. Given samples from a system,
we form an empirical covariance $C$ (or a regularized surrogate thereof), apply a power
transform $C^\beta$ for candidate $\beta\in\mathbb{R}$, and use $C^\beta$ as a
structure-informed feature matrix for a downstream supervised classifier. Notably, specific transforms like $\beta = -1$ or $1/2$ can be estimated directly, bypassing explicit covariance computation. Different values of $\beta$ interpolate between classical structural estimators. In particular, the precision matrix ($\beta=-1$) is structurally consistent when
samples are \emph{i.i.d.} and Gaussian~\cite{prob_graph_modes};
the square root of the precision ($\beta=-\tfrac{1}{2}$) is structurally consistent
for data generated by linear SEM;
and the covariance matrix itself ($\beta=1$) is structurally consistent for
ferromagnetic Ising models~\cite{Bento2009}.
We therefore select the exponent \emph{empirically}: the learned $\beta^\star$
is the one that best supports generalization on validation data via a proposed performance score.

\paragraph{Representation learning via a composite score.}
To choose $\beta^\star$ (and other hyperparameters such as window length), we
maximize a proposed composite score on a train--validation split while keeping the test set fully held out. The score is defined as:
\begin{equation}
{\sf S_3} \;=\;
4\frac{{\sf Spec_T}\,{\sf Spec_V}\,{\sf Sen_T}\,{\sf Sen_V}}
{{\sf Spec_T}+{\sf Spec_V}+{\sf Sen_T}+{\sf Sen_V}},
\label{eq:S3}
\end{equation}
where ${\sf Spec}$ and ${\sf Sen}$ denote specificity and sensitivity, and
subscripts $T,V$ refer to train and validation, respectively.
The product in the numerator is intended to heavily penalize low or unstable performance and favors
exponents that are simultaneously strong on both splits.
In the seizure setting, heterogeneity across patients motivates selecting
$\beta^\star$ \emph{per patient} using the same criterion.

\subsection{Data preparation}
\label{subsec:data}

\paragraph{EEG seizure onset (CHB-MIT).}
We use the CHB-MIT scalp EEG dataset (23 bipolar channels, 256 Hz), containing
24 pediatric cases with 19 seizures.
To mitigate extreme class imbalance without synthesizing data, we remove files
with no seizures, which reduces the dominance of interictal periods while
preserving seizure events. We apply a Butterworth bandpass filter to attenuate
high-frequency artifacts (e.g., line noise and physiological contamination).
To reduce redundancy across channels and focus on informative sensors, we perform
patient-specific channel selection using mutual information. Further, for train-validation, the time series 
are segmented using sliding windows with 75\% overlap as illustrated in Fig.~\ref{fig:slidewindow}. For test, each window is
labeled by majority vote over its samples. Window length is selected jointly with the exponent
$\beta^{\star}$ using ${\sf S_3}$ in~\eqref{eq:S3}.
For classification we use a 2D CNN followed by an LSTM.

\paragraph{Customer churn (IBM Telco).}
We use the IBM Telco churn dataset (7043 customers, 21 variables; $\approx$27\%
churn). Unlike EEG, each customer provides a single feature vector, yielding a single-shot regime. During training we form covariances over cohorts
(ensemble covariances). At test, we approximate a customer-level covariance
by a regularized rank-one cross-product $\mathbf{x}\mathbf{x}^{\top}$ (details in Appendix~\ref{app:frac-consistency}) to
enable one-shot inference with the same learned exponent $\beta^\star$. A standalone 2D CNN
is used as the downstream classifier.

\begin{figure}[t]
	\centering
	\includegraphics[scale=0.45]{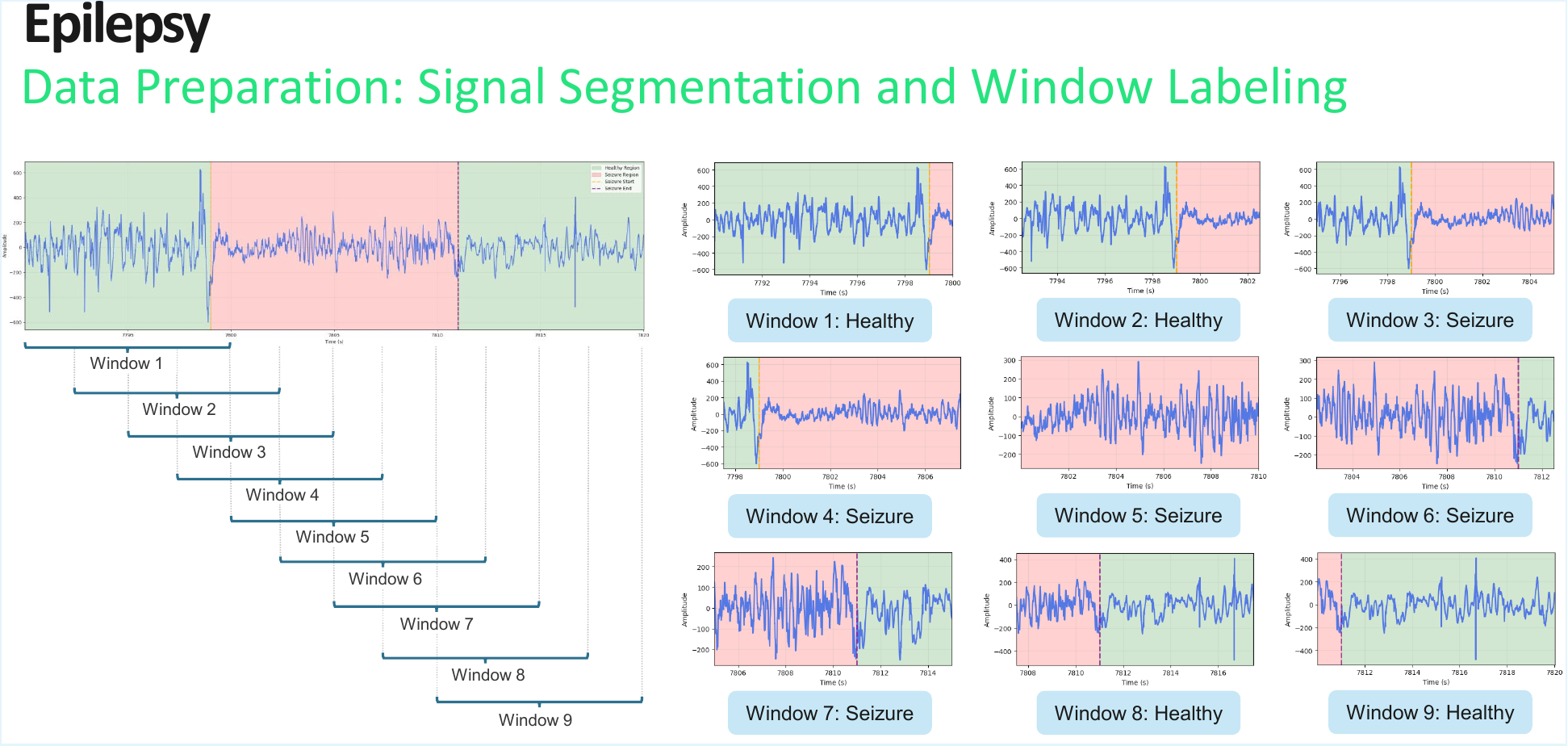}
	\caption{Sliding window to split the data for train-validation-test.}
	\label{fig:slidewindow}
\end{figure}

\section{Results}

We evaluate the proposed framework along three complementary axes: predictive performance, identifiability of the learned representations, and the potential presence of interpretable structural signatures. We select the main results and provide additional ones in the appendix~C.

\subsection{Predictive Performance}

Fig.~\ref{fig:benchmark1} reports seizure detection results
on the CHB-MIT dataset, benchmarking our approach against representative methods
from the literature that use the same data.

\begin{figure}[t]
	\centering
	\includegraphics[scale=0.45]{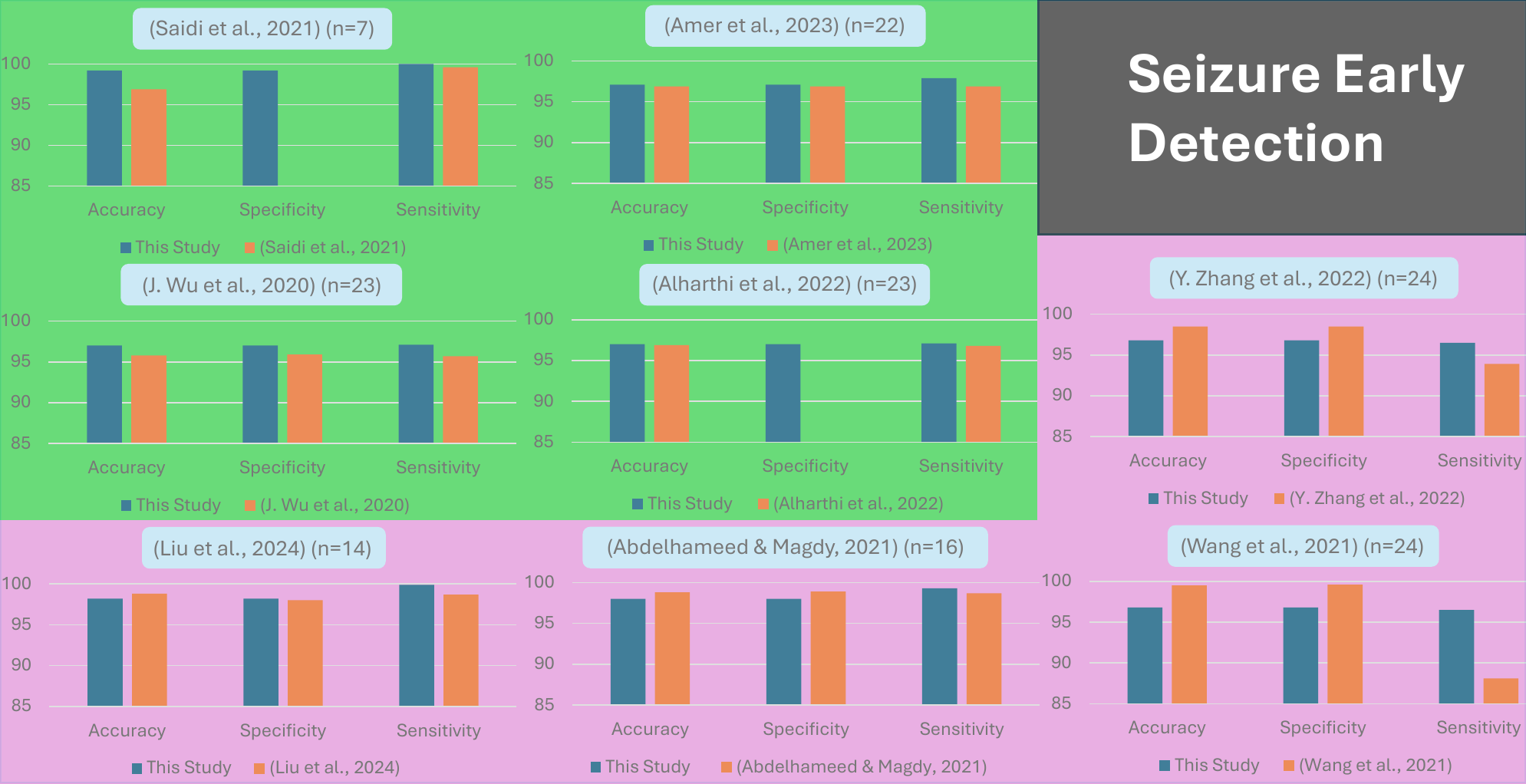}
	\caption{Seizure early detection benchmark: i) Superior performance across all metrics (green background); ii) Superior performance on sensitivity (purple background).}
	\label{fig:benchmark1}
\end{figure}

A critical aspect of these comparisons is the number of patients included in each
study. Several references report results on restricted subsets of patients,
often selecting cases that yield higher performance. In contrast, methods that
include all available patients provide a more realistic assessment of robustness
across a heterogeneous population.
To ensure fair comparison, each plot reports performance averaged over the top
$n_k$ patients, where $n_k$ matches the number of patients used in benchmark study
$k$, enabling direct comparison under matched cohort sizes.


As shown in Fig.~\ref{fig:benchmark1}, our approach outperforms benchmarks across all metrics (green background). While some methods (purple background) show marginal gains in accuracy or specificity, our model excels in the clinically critical metric of sensitivity. In this sense, our method demonstrates superior robustness and clinical viability.

\begin{figure}[t]
	\centering
	\includegraphics[scale=0.45]{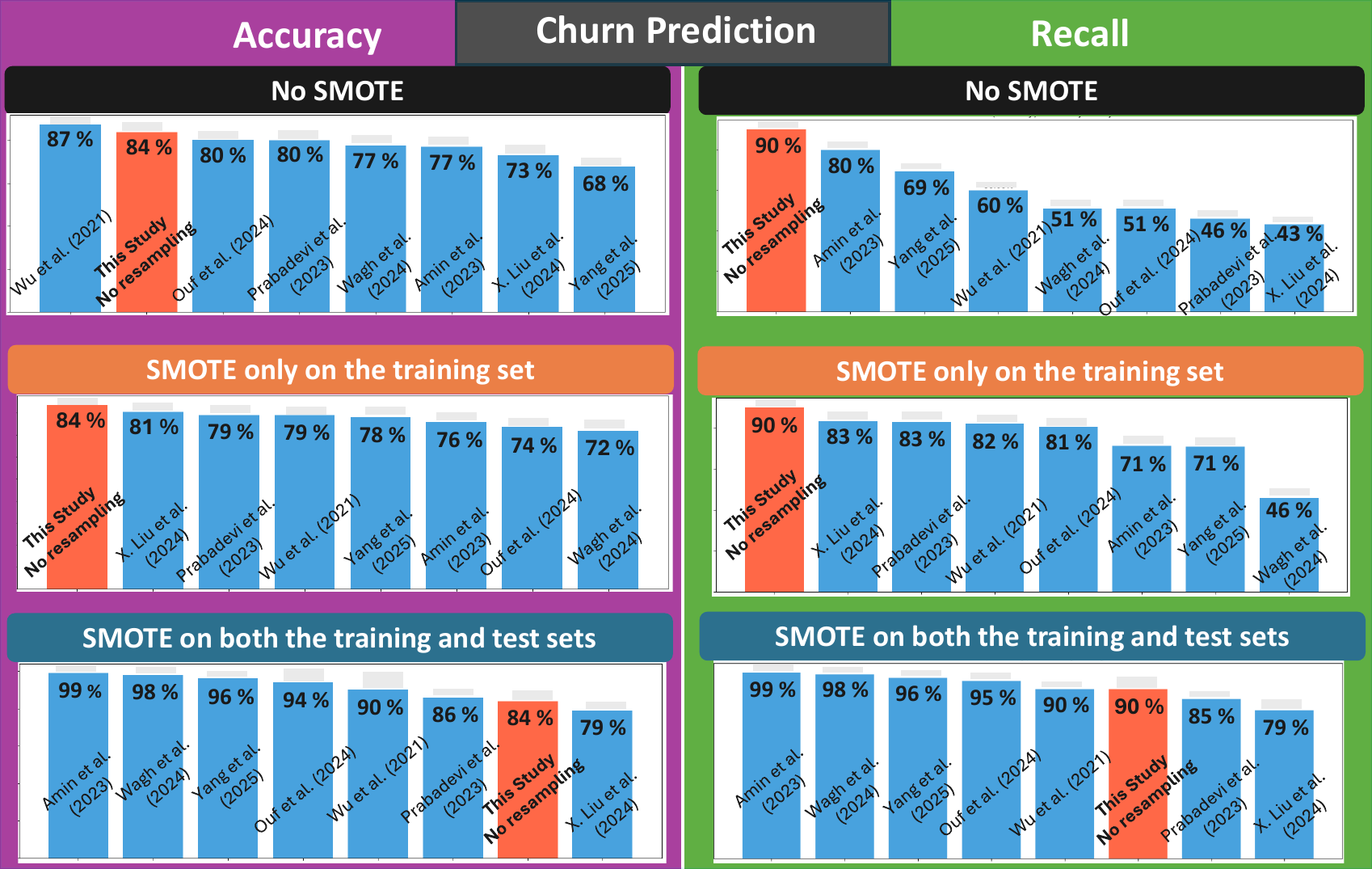}
	\caption{Accuracy and recall performance for churn prediction against benchmark references.}
	\label{fig:accuracy_churn}
\end{figure}


Similar trends appear in Fig.~\ref{fig:accuracy_churn} for customer churn. Unlike benchmarks that often rely on synthetic rebalancing—risking statistical leakage or inflated results—our method uses no artificial samples. Even when compared to benchmarks restricted to training-only rebalancing, our approach yields superior or comparable performance, notably in the business-critical metric of recall--critical for identifying
at-risk customers.

As shown in Fig.~\ref{fig:raw}, baseline models trained on raw data—using matched architectures and protocols—underperform significantly. This performance gap highlights the need of incorporating structural information in the feature design, as the models cannot sufficiently learn these representations from raw input alone. In appendix~C, we provide further experiments.

\begin{figure}[t]
	\centering
	\includegraphics[scale=0.5]{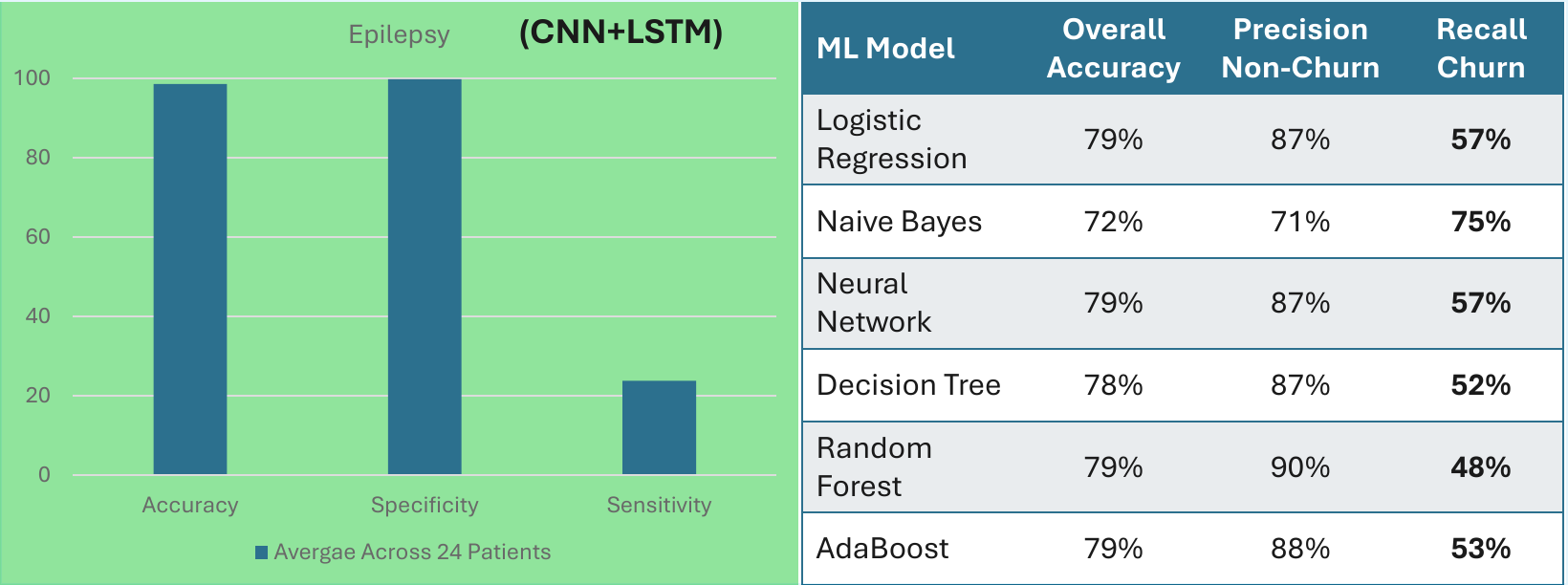}
	\caption{Performance of ML models for seizure detection
	(left) and churn prediction (right) trained directly on the raw data.}
	\label{fig:raw}
\end{figure}

Overall, these results demonstrate that structure-informed covariance power
representations yield strong and stable predictive performance across distinct
domains, without relying on task-specific heuristics or artificial data augmentation.

\subsection{Learned Feature Identifiability}

To explain the predictive performance of the learned representations, we investigate their \emph{identifiability}:
features corresponding to different classes should be
well separated, while features within the same class should form compact clusters.
Since covariance-power features are high-dimensional objects, direct visualization
is impractical. We therefore assess identifiability quantitatively via pairwise
distances.

Specifically, we compare:
(i) average distances between features belonging to the same class (intra-class),
and (ii) average distances between features belonging to different classes
(inter-class).
Because covariance matrices and their powers naturally live on the
Riemannian manifold of symmetric positive definite (SPD) matrices, distances are
computed using the affine-invariant Riemannian (AIR) metric.

Fig.~\ref{fig:id_seizures} reports these distances for the seizure detection task.
Across all patients, the inter-class distance (healthy vs.\ seizure, shown in
green) is consistently larger than the corresponding intra-class distances.
This clear separation indicates that the learned covariance-power features induce
well-separated class representations despite patient heterogeneity and partial
observability.
A similar pattern is observed for the variance of intra- and inter-class distances,
suggesting tighter clustering within classes than across classes.

\begin{figure}[t]
	\centering
	\includegraphics[scale=0.5]{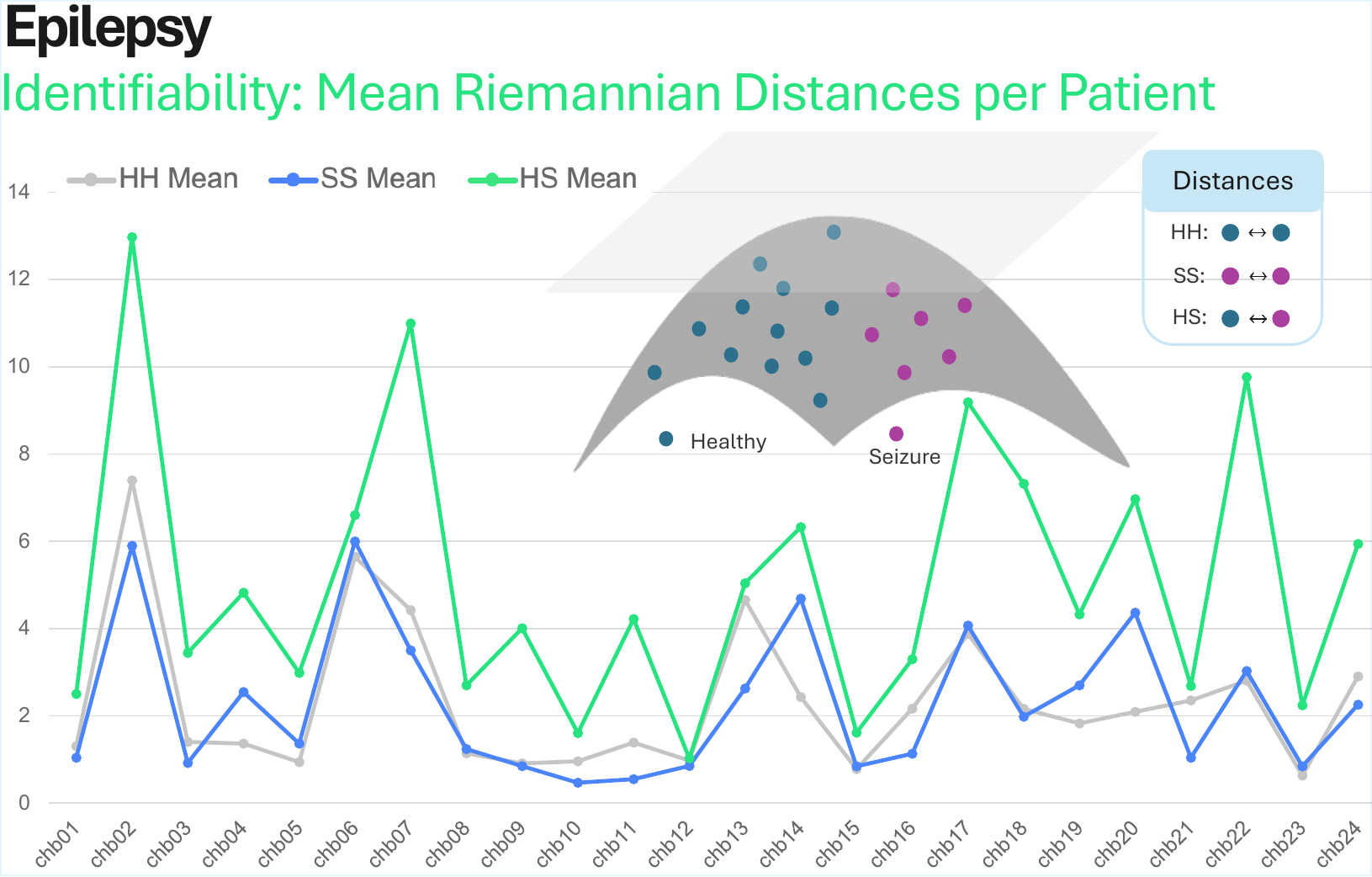}
	\caption{Average affine-invariant Riemannian distances between covariance-power
	features. Inter-class distances (green) are consistently larger than intra-class
	distances across all patients.}
	\label{fig:id_seizures}
\end{figure}

For customer churn, a similar separation is observed. The average AIR distances are
$14.7$ within the non-churn class, $11.18$ within the churn class, and $22.7$
between classes, again indicating that the learned representations are both
compact and discriminative.

Together, these results provide empirical evidence that the proposed
structure-informed features are not only predictive, but also yield identifiable
and clustered representations.

\subsection{Structural Signatures}

We now investigate whether the learned optimal
representations encode interpretable \emph{structural signatures} that reflect
distinct underlying interaction patterns across classes.
To this end, we analyze the covariance-power features corresponding to the
learned exponent $\beta^\star$ and attempt to extract a structural
backbone from them.

Specifically, we threshold the absolute values of the entries of the learned
feature matrices. Rather than fixing this threshold manually, we estimate it automatically
using a Gaussian mixture model (GMM) fitted to the empirical distribution of
feature entries magnitudes.
An illustration of this process for customer churn is shown in
Fig.~\ref{fig:struct_signatures}.

The resulting thresholded matrices can be interpreted as adjacency patterns
revealing class-specific interaction structures.
As shown on the right of Fig.~\ref{fig:struct_signatures}, the extracted graphs
exhibit markedly distinct connectivity patterns across classes, suggesting the
presence of characteristic structural identities. 


\begin{figure}[t]
	\centering
	\includegraphics[scale=0.5]{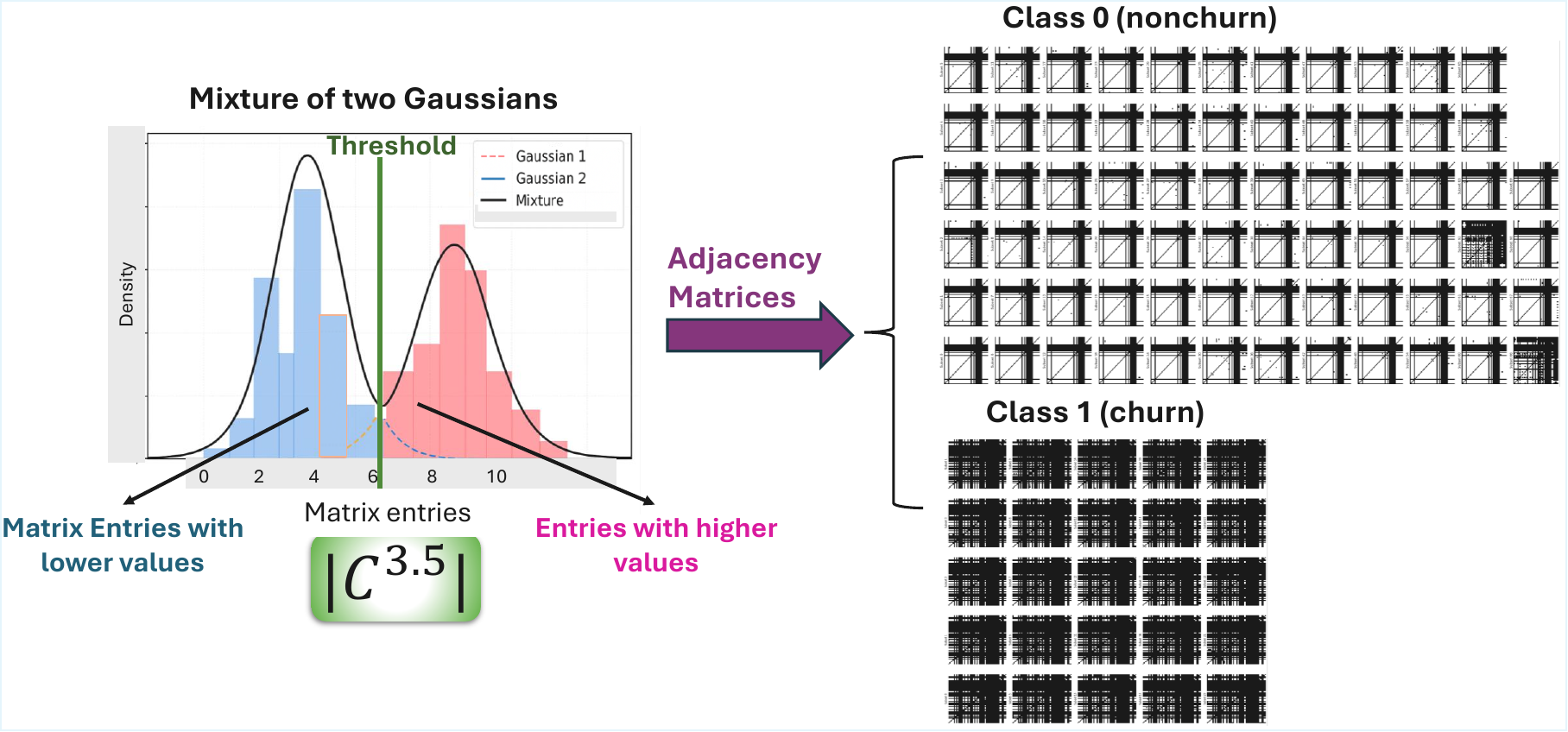}
	\caption{Distilling structural signatures from churn data via GMM-based
	thresholding. Distinct patterns emerge for different classes.}
	\label{fig:struct_signatures}
\end{figure}

We observe a similar phenomenon in the seizure detection setting.
Fig.~\ref{fig:struct_signatures_seiz} displays a representative example for a
single patient, where the learned structural signature highlights coherent
connectivity patterns that differ markedly between interictal and preictal
states.

\begin{figure}[t]
	\centering
	\includegraphics[scale=0.7]{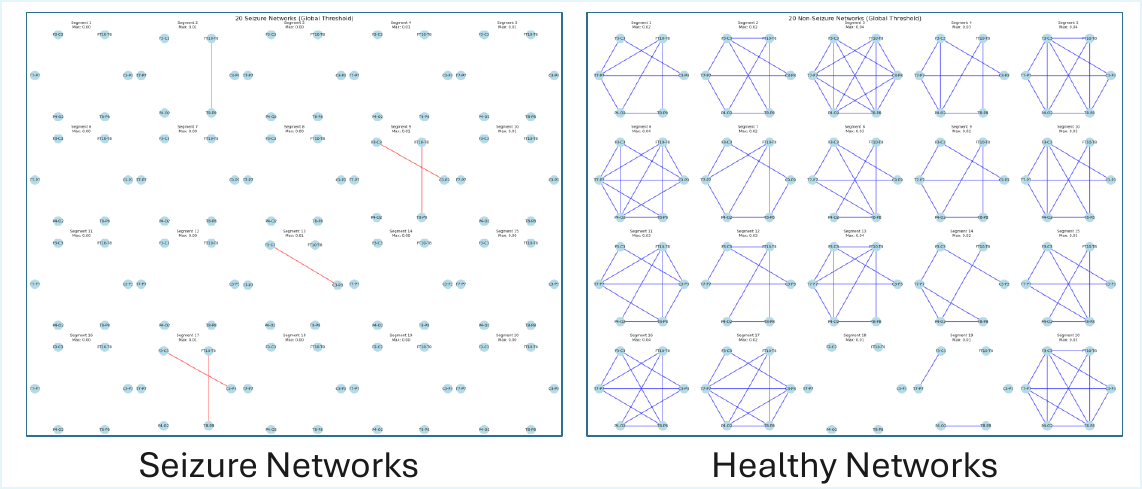}
	\caption{Representative structural signatures extracted for seizure detection
	(CHB-MIT, patient chb19).}
	\label{fig:struct_signatures_seiz}
\end{figure}

While our primary objective is not exact graph recovery, these results provide
qualitative empirical evidence that the covariance-power representations capture
non-trivial and class-dependent interaction structure.


\section{Concluding Remarks}

Beyond the specific applications considered here, our results point to a broader
principle: emergent phenomena across heterogeneous domains may be detected by
learning representations that align with latent interaction structure rather than
task-specific signal characteristics alone.
The same covariance–power framework applies seamlessly to biological,
socio-economic, and technological systems, despite their distinct data modalities
and generative mechanisms.
This apparent universality is not imposed by design, but rather emerges as a
consequence of the structural consistency guarantees established in this work and
their empirical realization.
We believe this perspective opens new avenues for adaptive structure-aware learning in
settings where the underlying dynamics are only partially observed, weakly
specified, or fundamentally unknown.

\section*{Acknowledgments}
Work partially supported by NSF Grant CCF-2327905.

\appendix
\onecolumn

\section{Preliminary Formalism}
\label{sec:oscandall}
In this section, we establish notation and collect a few auxiliary bounds that will be
invoked throughout the paper. This separates technical conventions from the main
development, and ensures that all subsequent derivations can be followed more transparently.

\subsection{Notation}

Vectors and random variables are rendered in bold lower-case, whereas matrices are represented with capital letters (and bold if random). Calligraphic letters are reserved for sets, except for standard sets such as the real line $\mathbb{R}$. The notation $\big|\mathcal{S}\big|$ denotes the cardinality of the set $\mathcal{S}$.

\paragraph{Norms and eigenvalues.}
For a general matrix $B\in\mathbb{R}^{N\times N}$,
we denote by
\begin{equation}
\|B\| \;=\; \max_{\|\mathbf{u}\|_2=1,\ \|\mathbf{v}\|_2=1}
\big|\mathbf{u}^{\top} B \mathbf{v}\big| = \max_{\|\mathbf{x}\|_2=1}
\|B \mathbf{x}\|_2 
\end{equation}
the spectral norm (or operator $2$-norm) of $B$. We refer to $\sigma(B):=\left\{\lambda_1(B),\lambda_2(B),\ldots,\lambda_N(B)\right\}$ as the spectrum of $B$, where $\lambda_i(B)$ is the $i$th eigenvalue of $B$.
The smallest eigenvalue of $B$ is written as $\lambda_{\min}(B)=\min\left\{\lambda\,:\,\lambda \in \sigma(B)\right\}$.
The spectral radius of $B$ is $\rho(B)=\max\left\{|\lambda|\,:\,\lambda \in\ \sigma(B)\right\}$.
For general $B\in \mathbb{R}^{N\times N}$, we have $\rho(B)\leq\|B\|$ with equality holding whenever $B$ is symmetric, $\rho(B)=\|B\|$.
A matrix is said to be \emph{stable} if $\rho(B)<1$.
When the underlying matrix is clear from the context, we simply write it as $\rho$.

\paragraph{Identity, canonical basis, and all-ones vector.}
$I_N\in\mathbb{R}^{N\times N}$ denotes the $N\times N$ identity matrix
(we drop the subscript when the dimension is clear).
$\mathbf{e}_i\in\mathbb{R}^N$ denotes the $i$th canonical basis vector, i.e., all entries zero except for the $i$th entry, which equals one.
$\mathbf{1}_N\in\mathbb{R}^N$ is the vector of all ones
(again, we drop the subscript if the dimension is clear). 
Observe that $\mathbf{1}_{N}\mathbf{1}_N^{\top}$ is the $N\times N$ matrix of all ones. 

\paragraph{Projections and submatrices.}
Given an index subset $\mathcal{S}=\{s_1,s_2,\dots,s_{|\mathcal{S}|}\}\subset\{1,2,\dots,N\}$,
we denote by $\Pi_{\mathcal{S}}\in\mathbb{R}^{|\mathcal{S}|\times N}$
the coordinate projection matrix.
For any vector $\mathbf{v}\in\mathbb{R}^N$, the projection is
$\Pi_{\mathcal{S}}\mathbf{v}=\mathbf{v}_{\mathcal{S}}
=(v_{s_1},\dots,v_{s_{|\mathcal{S}|}})$.
For a matrix $B\in\mathbb{R}^{N\times N}$, the principal submatrix supported on $\mathcal{S}$ is
\begin{equation}
[B]_{\mathcal{S}} = \Pi_{\mathcal{S}} B \Pi_{\mathcal{S}}^{\top}.
\end{equation}
For simplicity, we sometimes write it as $B_{\mathcal{S}}$.

\subsection{Auxiliary Bounds}

We now record two simple bounds that will be repeatedly used. First, we define an important map tied with the concept of structural consistency of a matrix-valued estimator.

\paragraph{Off--diagonal oscillation.}
For a matrix $M\in\mathbb{R}^{N\times N}$, define the \emph{off-diagonal oscillation} map $\operatorname{Osc}\,:\,\mathbb{R}^{N\times N}\longrightarrow \mathbb{R}_{+}$
\begin{equation}
\operatorname{Osc}(M) := \max_{i\neq j} M_{ij}\ -\ \min_{i\neq j} M_{ij}.
\end{equation}
In words, $\operatorname{Osc}(M)$ measures the gap between the greatest and the smallest off-diagonal entries of the matrix $M$, i.e., the \emph{spread} or \emph{flatness} of the off-diagonal submatrix component. For example, $\operatorname{Osc}(\beta\mathbf{1}\mathbf{1}^{\top}+D) = 0$, for any $\beta\in\mathbb{R}$, where $D$ is any arbitrary diagonal matrix. We remark that the map $\operatorname{Osc}\,:\,\mathbb{R}^{N\times N}\rightarrow \mathbb{R}_{+}$ is invariant under \emph{flat-shifts} 
\begin{equation}
\operatorname{Osc}\left(M+\beta \mathbf{1}\mathbf{1}^{\top}+D\right)=\operatorname{Osc}\left(M\right), 
\end{equation}
where $D$ is again any diagonal matrix. In other words, $\operatorname{Osc}(M)$ of a matrix $M$ is invariant upon uniformly shifting all of its off-diagonal entries. Refer to~\cite{SantosA,SMachado} for further properties of this map.

\textbf{Example.} \emph{To illustrate, consider the matrix
\begin{equation}
M = \begin{bmatrix}
1 & 2 & 5 \\
3 & 4 & 6 \\
7 & 8 & 9
\end{bmatrix}.
\end{equation}
The off--diagonal entries are $\{2, 3, 5, 6, 7, 8\}$. The greatest off--diagonal entry is $ \max_{i\neq j} M_{ij} = 8$ and the smallest is $ \min_{i\neq j} M_{ij}= 2$.
Thus,
\begin{equation}
\operatorname{Osc}(M) = \max_{i\neq j} M_{ij}\ -\ \min_{i\neq j} M_{ij} = 8 - 2 = 6.
\end{equation}
Now, apply a flat-shift to the off-diagonals of $M$ by $\beta=10$ (and arbitrary $D$), yielding
\begin{equation}
\widetilde{M} = M + \beta\mathbf{1}\mathbf{1}^{\top}+D = \begin{bmatrix}
111 & 12 & 15 \\
13 & 12 & 16 \\
17 & 18 & 59
\end{bmatrix}.
\end{equation}
The greatest off-diagonal entry of $\widetilde{M}$ is $\max_{i\neq j} \widetilde{M}_{ij} = 18$ and the smallest is $\min_{i\neq j} \widetilde{M}_{ij} = 12$.
Thus,
\begin{equation}
\begin{array}{ccl}
\operatorname{Osc}(\widetilde{M}) & = & \max_{i\neq j} \widetilde{M}_{ij}\ -\ \min_{i\neq j} \widetilde{M}_{ij}\\
& = & 18 - 12 \\
& = & \max_{i\neq j} (M_{ij}+\beta) -\ \min_{i\neq j} (M_{ij}+\beta)\\
& = & \max_{i\neq j} M_{ij} -\ \min_{i\neq j} M_{ij}\\
& = & \operatorname{Osc}(M) = 6
\end{array}.
\end{equation}
As expected, $\operatorname{Osc}(\widetilde{M})=\operatorname{Osc}(M)$.}

\paragraph{Spectral norm control.}
Two foundational bounds will be invoked momentarily
\begin{equation}\label{eq:entry-norm}
\max_{i\neq j} |M_{ij}| \;\le\; \|M\|,
\end{equation}
and
\begin{equation}\label{eq:Osc-norm}
\operatorname{Osc}(M+\beta \mathbf{1}\mathbf{1}^{\top}+D)\ \le\ 2\,\|M\|,
\end{equation}
for any $M\in\mathbb{R}^{N\times N}$, scalar $\beta\in\mathbb{R}$, and diagonal matrix $D\in\mathbb{R}^{N\times N}$.

Indeed, the first inequality~\eqref{eq:entry-norm}, ascertaining that the off-diagonal entries are bounded by the spectral norm, follows from the definition of spectral norm
\begin{equation}\label{eq:break_entry_norm}
|M_{ij}| = |\mathbf{e}_i^\top M \mathbf{e}_j|\leq \max_{\|\mathbf{u}\|=1,\|\mathbf{v}\|=1} \left|\mathbf{u}^{\top}M \mathbf{v}\right| = \|M\|,
\end{equation}
for all $i\neq j$, whereas the second~\eqref{eq:Osc-norm} resorts to this inequality to conclude that 
\begin{equation}
\begin{array}{ccl}
\operatorname{Osc}(M+\beta \mathbf{1}\mathbf{1}^{\top}+D) & \overset{(a)}= & \operatorname{Osc}(M)\\
& &\\
& \overset{(b)}= &\max_{i\neq j} M_{ij}\ -\ \min_{i\neq j} M_{ij}\\
& & \\
& = & \max_{i\neq j} M_{ij}\ + \max_{i\neq j} -M_{ij}\\
& & \\
& \leq & 2\max_{i\neq j} |M_{ij}|  \overset{(c)}\leq 2 \|M\|
\end{array},
\label{eq:norm_osc}
\end{equation}
where the identity $(a)$ follows from the invariance of $\operatorname{Osc}$ under flat-shifts, the identity $(b)$ conforms to the definition of $\operatorname{Osc}$, and the latter inequality $(c)$ holds in view of~\eqref{eq:entry-norm}.

These conventions and inequalities will be used throughout the paper without further comment.

\newcommand{\bbR}{\mathbb{R}}
\newcommand{\osc}{\mathrm{osc}}
\newcommand{\I}{I}

%
%

\section{Structural Consistency for Fractional Fields}
\label{app:frac-consistency}

In this section, we establish structural consistency for the fractional
inverse of the observed covariance for a general Matérn--type model~\eqref{eq:GraphMatern} with covariance
\begin{equation}
C \;=\; \sigma^2 \left(I - \overline{A}\right)^{-\alpha},
\qquad
\overline{A} = \overline{A}^\top,\quad \rho(\overline{A}) < 1,\quad \alpha\neq 0,
\label{eq:frac-model-appendix}
\end{equation}
where $\rho(\overline{A})$ denotes the spectral radius of $\overline{A}$ and $\overline{A}:= I-\kappa^2 D - L$, i.e., the parametric details of the model~\eqref{eq:GraphMatern} are absorbed in $\overline{A}$. Remark that the off-diagonal elements of $\overline{A}$ coincide with those of $A$, i.e., ${\sf Off}\left(\overline{A}\right)={\sf Off}\left(A\right)$, where $A$ is the ground truth interaction matrix defining the Laplacian $L=D-A$. Therefore, the graph support of $\overline{A}$ (without the self-loops) coincides with that of the underlying unknown ground truth graph structure (support of $A$). 

We partition the node set as $\{1,\dots,N\} = \mathcal S \cup \mathcal S'$,
$\mathcal S\cap\mathcal S'=\emptyset$, and write
\[
A =
\begin{bmatrix}
A_{\mathcal S} & A_{\mathcal S\mathcal S'}\\
A_{\mathcal S'\mathcal S} & A_{\mathcal S'}
\end{bmatrix},
\qquad
C =
\begin{bmatrix}
C_{\mathcal S} & C_{\mathcal S\mathcal S'}\\
C_{\mathcal S'\mathcal S} & C_{\mathcal S'}
\end{bmatrix},
\]
where $\mathcal{S}$ stands for the set of observed nodes and $\mathcal{S}'$ and the complement corresponds to the latent counterpart. 

Our goal is to prove that $(C_{\mathcal S})^{-1/\alpha}$ is structurally consistent: all entries of the matrix $(C_{\mathcal S})^{-1/\alpha}$ associated with disconnected pairs lie below the smallest entry associated with connected pairs. The importance of the consistency property is twofold: i) it means in that we can recover the support of $A_{\mathcal S}$ via appropriately thresholding the entries of $(C_{\mathcal S})^{-1/\alpha}$; ii) full information about the network structure linking the observed nodes is entailed in the estimator~$(C_{\mathcal S})^{-1/\alpha}$ (despite the presence of latent confounders) and therefore, this is a potential candidate for structure-informed feature in machine learning pipelines for inference in partially observed systems (virtually the case for most complex systems).

As we will establish, this structural consistency holds under certain conditions on the cross--block $A_{\mathcal S\mathcal S'}$ linking the observed and the latent parts. We start by establishing the result for powers $\beta\in\left(0,1\right)$ within the unit interval. Then, in Appendix~\ref{app:frac-consistency-dt}, we extend our result to general powers $\beta\in \mathbb{R}$. These will demand distinct sufficient conditions on $A_{\mathcal S\mathcal S'}$.

\subsection{Fractional Integral Representation: $\beta\in\left(0,1\right)$.}
Let $\beta := 1/\alpha \in (0,1)$.
For any Symmetric Positive-Definite (SPD) matrix $X\succ 0$, the fractional inverse admits the integral representation
\begin{equation}
X^{-\beta}
=
\frac{\sin(\pi\beta)}{\pi}
\int_0^\infty \lambda^{-\beta}
(X + \lambda I)^{-1}\, d\lambda.
\label{eq:stieltjes}
\end{equation}
The integral expression~\eqref{eq:stieltjes} offers a representation of a fractional power of a matrix (or operator) in terms of the inverse of the matrix. This will be useful for our purposes. 

Define the error between the fractional inverse of the observed block and
the corresponding principal submatrix of the full fractional inverse:
\begin{equation}
\Delta_{-\beta}
:= (C_{\mathcal S})^{-\beta} - [C^{-\beta}]_{\mathcal S},
\label{eq:gap}
\end{equation}
which can be cast as the error matrix committed when computing $(C_{\mathcal S})^{-\beta}$ (the covariance power computed upon partially observed samples) instead of $[C^{-\beta}]_{\mathcal S} = \sigma^{-2/\alpha}\left(I-\overline{A}_{\mathcal S}\right)$ which contains full structural information since 
\begin{equation}
-{\sf Off}\left([C^{-\beta}]_{\mathcal S}\right)= {\sf Off}\left(\overline{A}_{\mathcal{S}}\right)={\sf Off}\left(A_{\mathcal{S}}\right),
\end{equation}
but cannot be estimated under partial observability as to estimate the full matrix $C$ we need all observables.

To grant structural consistency of~$(C_{\mathcal S})^{-\beta}$ (the one we can estimate from the observed samples), we need the error matrix term to be small enough, namely~\cite{SantosA}
\begin{equation}
{\sf Osc}\left(\Delta_{-\beta}\right)<\frac{a_{\min}}{2\sigma^{2\beta}},
\label{eq:thresh}
\end{equation}
where $a_{\min}$ is the smallest non-zero off-diagonal entry of $A_{\mathcal{S}}$, i.e.,
\begin{equation}
a_{\min}
:=
\min\{A_{ij}>0:\,i\neq j,\ i,j\in\mathcal S\},
\end{equation}
and ${\sf Osc}$ is defined in Appendix~\ref{sec:oscandall}. The condition~\eqref{eq:thresh} ascertains that even though the error perturbs the entries of~$[C^{-\beta}]_{\mathcal S}$ (which fully contains the network structure), the structure is preserved in $(C_{\mathcal S})^{-\beta}$ (and not destroyed), namely, all entries of $(C_{\mathcal S})^{-\beta}$ associated with disconnected pairs are smaller than the smallest entry associated with a connected pair. In other words, there is a threshold to consistently cluster disconnected from connected pairs in view of the entries of $(C_{\mathcal S})^{-\beta}$. Another way to rephrase it: under condition~\eqref{eq:thresh}, projection and matrix exponentiation commute in a structural sense--they both entail the same structural information even though their entries may be distinct since $(C_{\mathcal S})^{-\beta}\neq \left[C^{-\beta}\right]_{\mathcal S}$.     

Therefore, to grant structural consistency of~$(C_{\mathcal S})^{-\beta}$, ${\sf Osc}\left(\Delta_{-\beta}\right)$ needs to be small enough. It is thus important to control the \emph{magnitude} of~$\Delta_{-\beta}$ and we do it via introducing the resolvent difference
\begin{equation}
R_{\mathcal S}(\lambda)
:=
(C_{\mathcal S}+\lambda I)^{-1}
-
[(C+\lambda I)^{-1}]_{\mathcal S}.
\label{eq:gap2}
\end{equation}
Now, in view of the integral characterization~\eqref{eq:stieltjes} and the definitions~\eqref{eq:gap}-\eqref{eq:gap2}, we have
\begin{equation}
\Delta_{-\beta}
=
\frac{\sin(\pi\beta)}{\pi}
\int_0^\infty \lambda^{-\beta}
R_{\mathcal S}(\lambda)\,d\lambda,
\end{equation}
and therefore
\begin{equation}
\|\Delta_{-\beta}\|
\;\le\;
\frac{\sin(\pi\beta)}{\pi}
\int_0^\infty \lambda^{-\beta}
\|R_{\mathcal S}(\lambda)\|\,d\lambda,
\label{eq:delta-beta-bound}
\end{equation}
where $\|\cdot\|$ is the $L_2$ operator norm and we recall that $\beta\in\left(0,1\right)$ so that $\sin\left(\beta \pi\right)>0$.

The inequality~\eqref{eq:delta-beta-bound} allows us to control $\|\Delta_{-\beta}\|$ (so to yield structural consistency) via controlling the resolvent $\|R_{\mathcal S}(\lambda)\|$ which is a construct with fixed power (namely the inverse) and can be bounded via Schur complement as we will see next.

\paragraph{Schur Complement and Neumann Expansion.}
We now bound $R_{\mathcal S}(\lambda)$ via characterizing the underlying Schur complement and then, bounding the Neumann
series.

For convenience, and for $\lambda>0$, let us write
\[
C + \lambda I
=
\begin{bmatrix}
B(\lambda) & C_{\mathcal S\mathcal S'}\\
C_{\mathcal S'\mathcal S} & D(\lambda)
\end{bmatrix},
\qquad
B(\lambda) := C_{\mathcal S} + \lambda I,\ 
D(\lambda) := C_{\mathcal S'} + \lambda I.\ 
\]
Schur complement~\cite{MatrixAnalysis} yields
\begin{equation}
[(C+\lambda I)^{-1}]_{\mathcal S}
=
(B(\lambda) - H(\lambda))^{-1},
\end{equation}
where we have defined $H(\lambda) := C_{\mathcal{S}\mathcal{S}'} D(\lambda)^{-1} C_{\mathcal{S}'\mathcal{S}}$. Thus,
\begin{equation}
R_{\mathcal S}(\lambda)
=
B(\lambda)^{-1} - (B(\lambda) - H(\lambda))^{-1}
=
-(B(\lambda)-H(\lambda))^{-1} H(\lambda) B(\lambda)^{-1}.
\label{eq:RS_Schur}
\end{equation}

We now derive a general bound on the $L_2$-norm of $R_{\mathcal S}(\lambda)$ that will be useful to bound the error term $\Delta_{-\beta}$ and thus, for establishing structural consistency of $\left(C_{\mathcal{S}}\right)^{-\beta}$.

\begin{lemma}[Bound on $R_{\mathcal S}(\lambda)$]
\label{lem:R-bound}
Assume $C_{\mathcal S},C_{\mathcal S'}\succ 0$ and define
\[
\theta :=
\frac{\|C_{\mathcal S\mathcal S'}\|^2}
{\lambda_{\min}(C_{\mathcal S})\,\lambda_{\min}(C_{\mathcal S'})}.
\]
If $\theta<1$, then for all $\lambda>0$,
\begin{equation}
\|R_{\mathcal S}(\lambda)\|
\;\le\;
\frac{\|C_{\mathcal S\mathcal S'}\|^2}
{(1-\theta)\,(\lambda_{\min}(C)+\lambda)^3},
\end{equation}
where $\lambda_{\min}(C)$ denotes the smallest eigenvalue of the correlation matrix~$C$.
\end{lemma}

\begin{proof}
First note that
\[
B(\lambda) \succeq (\lambda_{\min}(C_{\mathcal S})+\lambda)I,
\quad
D(\lambda) \succeq (\lambda_{\min}(C_{\mathcal S'})+\lambda)I,
\]
for all $\lambda>0$, so that
\begin{equation}
\|B(\lambda)^{-1}\| \le \frac{1}{\lambda_{\min}(C_{\mathcal S})+\lambda},
\qquad
\|D(\lambda)^{-1}\| \le \frac{1}{\lambda_{\min}(C_{\mathcal S'})+\lambda},
\label{eq:D_char}
\end{equation}
for all $\lambda > 0$. Moreover, in view of equation~\eqref{eq:D_char}, the symmetry of the correlation matrix $C$, namely, $C_{\mathcal{S} \mathcal{S}'} = C_{\mathcal{S}' \mathcal{S}}^{\top}$, and the submultiplicativity of $\|\cdot\|$, we have
\[
\|H(\lambda)\|
=
\|C_{\mathcal{S}\mathcal{S}'} D(\lambda)^{-1}C_{\mathcal{S}'\mathcal{S}}^\top\|
\le
\frac{\|C_{\mathcal{S}\mathcal{S}'}\|^2}{\lambda_{\min}(C_{\mathcal S'})+\lambda}.
\]
Hence
\begin{equation}
\|B(\lambda)^{-1}H(\lambda)\|
\le
\frac{\|C_{\mathcal{S}\mathcal{S}'}\|^2}
{(\lambda_{\min}(C_{\mathcal S})+\lambda)(\lambda_{\min}(C_{\mathcal S'})+\lambda)}.
\label{eq:BH}
\end{equation}
Now, define
\[
\theta :=
\frac{\|C_{\mathcal{S}\mathcal{S}'}\|^2}
{\lambda_{\min}(C_{\mathcal S})\,\lambda_{\min}(C_{\mathcal S'})}.
\]
Clearly, from inequality~\eqref{eq:BH}, we have that~$\sup_{\lambda>0}\|B(\lambda)^{-1}H(\lambda)\|\le \theta$. Therefore,
if $\theta<1$, then the following Neumann series converges for all $\lambda>0$
\[
(B(\lambda)-H(\lambda))^{-1}
= (I-B(\lambda)^{-1}H(\lambda))^{-1}B(\lambda)^{-1}
= \sum_{m\ge 0} (B(\lambda)^{-1}H(\lambda))^m B(\lambda)^{-1}.
\]
Thus, from the characterization of $R_{\mathcal{S}}(\lambda)$ in equation~\eqref{eq:RS_Schur}, we have
\begin{align}
\|R_{\mathcal S}(\lambda)\|
&=
\|(B(\lambda)-H(\lambda))^{-1}H(\lambda)B(\lambda)^{-1}\|\\
&\le
\sum_{m\ge 0} \|B(\lambda)^{-1}H(\lambda)\|^m\,\|B(\lambda)^{-1}\|^2\,\|H(\lambda)\|\\
&\le
\frac{\|H(\lambda)\|\,\|B(\lambda)^{-1}\|^2}{1-\theta}.
\end{align}
Using the bounds above and $\lambda_{\min}(C)\le\min\left\{\lambda_{\min}(C_{\mathcal S}),\lambda_{\min}(C_{\mathcal S'})\right\}$, we obtain
\[
\|R_{\mathcal S}(\lambda)\|
\le
\frac{\|C_{\mathcal{S}\mathcal{S}'}\|^2}
{(1-\theta)\,(\lambda_{\min}(C)+\lambda)^3}.
\]
\end{proof}

\paragraph{Bounding $\Delta_{-\beta}$ via a Beta Integral.}
Combining~\eqref{eq:delta-beta-bound} with Lemma~\ref{lem:R-bound} we obtain the following bound
\[
\|\Delta_{-\beta}\|
\le
\frac{\sin(\pi\beta)}{\pi}\,
\frac{\|C_{\mathcal S\mathcal S'}\|^2}{1-\theta}
\int_0^\infty
\frac{\lambda^{-\beta}}
{(\lambda_{\min}(C)+\lambda)^3}\,d\lambda.
\]
The above integral representation can be written in terms of a beta integral. Namely, let us change variables $\lambda := \lambda_{\min}(C)t$,
\[
\int_0^\infty
\frac{\lambda^{-\beta}}{(\lambda_{\min}(C)+\lambda)^3}\,d\lambda
=
\lambda_{\min}(C)^{-2-\beta}
\int_0^\infty \frac{t^{-\beta}}{(1+t)^3}\,dt.
\]
The latter integral is a beta integral:
\[
\int_0^\infty \frac{t^{-\beta}}{(1+t)^3}\,dt
=
B(1-\beta,2+\beta)
=
\frac{\Gamma(1-\beta)\Gamma(2+\beta)}{2}.
\]
Thus,
\begin{equation}
\|\Delta_{-\beta}\|
\;\le\;
\frac{\Gamma(1-\beta)\Gamma(2+\beta)\sin(\pi\beta)\|C_{\mathcal S\mathcal S'}\|^2}
{2\pi(1-\theta)\,\lambda_{\min}(C)^{2+\beta}},
\label{eq:delta-beta-C-bound}
\end{equation}
with $\beta\in\left(0,1\right)$.

\paragraph{Relating $C$ and $A$.}
The bound~\eqref{eq:delta-beta-C-bound} on the $L_2$ norm of $\Delta_{-\beta}$ is expressed in terms of the covariance matrix and it is useful to have a characterization in terms of the ground truth interaction matrix $A$. Thus, we now express $\lambda_{\min}(C)$ and $\|C_{\mathcal S\mathcal S'}\|$ in terms of
the interaction matrix $A$, i.e., the matrix conveying the ground truth network structure in its support. This will permit characterizing the gap error $\Delta_{-\beta}$ in terms of the structural matrix $A$.

\begin{lemma}[Spectral relation between $C$ and $A$ or $\overline{A}$]
\label{lem:spec-relation}
Recall that $C=\sigma^2 (I-\overline{A})^{-\alpha}$ with $\overline{A}=\overline{A}^\top$ and $\rho(\overline{A})<1$.
Then,
\begin{enumerate}
\item The smallest eigenvalue of $C$ is given by
\[
\lambda_{\min}(C)
=
\sigma^2 \left(1-\lambda_{\min}\left(\overline{A}\right)\right)^{-\alpha}.
\]
\item The cross--block covariance obeys
\[
\|C_{\mathcal S\mathcal S'}\|
\;\le\;
\sigma^2\,\alpha\,(1-\rho(\overline{A}))^{-\alpha-1}\,\|A_{\mathcal S\mathcal S'}\|.
\]
\end{enumerate}
\end{lemma}

\begin{proof}
Diagonalize $\overline{A}=Q\Lambda Q^\top$ with $\Lambda=\mathrm{diag}(\lambda_1,\dots,\lambda_N)$.
Then
\[
C = \sigma^2 (I-\overline{A})^{-\alpha}
= \sigma^2 Q(I-\Lambda)^{-\alpha}Q^\top.
\]
The eigenvalues of $C$ are $\sigma^2(1-\lambda_i)^{-\alpha}$, so
\[
\lambda_{\min}(C)
=
\sigma^2 \min_i (1-\lambda_i)^{-\alpha}
=
\sigma^2 (1-\lambda_{\min}\left(\overline{A}\right))^{-\alpha}.
\]

To characterize the cross--block $C_{\mathcal{S}\mathcal{S}'}$, we resort to the generalized binomial series
\[
(I-\overline{A})^{-\alpha}
=
\sum_{k=0}^\infty c_k(\alpha) \overline{A}^k,
\qquad
c_k(\alpha)
=
\frac{\Gamma(\alpha+k)}{\Gamma(\alpha)k!},
\]
valid since $\rho(\overline{A})<1$. Hence
\[
C
=
\sigma^2 \sum_{k=0}^\infty c_k(\alpha) \overline{A}^k
\quad\Rightarrow\quad
C_{\mathcal S\mathcal S'}
=
\sigma^2 \sum_{k\ge 1} c_k(\alpha)[\overline{A}^k]_{\mathcal S\mathcal S'},
\]
We use the bound on cross--blocks of powers of $A$ in the auxiliary Lemma~\ref{lem:Ak-SSp} (below), namely, write $\rho:=\|\overline{A}\|$,
\[
\|C_{\mathcal S\mathcal S'}\|
\le
\sigma^2 \|A_{\mathcal S\mathcal S'}\|\sum_{k\ge 1} k\,c_k(\alpha)\,\rho^{k-1},
\]
where we recall that~$\overline{A}_{\mathcal S\mathcal S'} = A_{\mathcal S\mathcal S'}$. Now, consider the scalar generating function
\[
f_\alpha(x) := (1-x)^{-\alpha}
= \sum_{k=0}^\infty c_k(\alpha)x^k,\quad |x|<1.
\]
Differentiating,
\[
f_\alpha'(x)
= \sum_{k\ge 1} k\,c_k(\alpha)\,x^{k-1}
= \alpha (1-x)^{-\alpha-1}.
\]
Evaluating at $x:=\rho$ gives
\[
\sum_{k\ge 1} k\,c_k(\alpha)\,\rho^{k-1}
= \alpha (1-\rho)^{-\alpha-1},
\]
and hence
\[
\|C_{\mathcal S\mathcal S'}\|
\le
\sigma^2\,\alpha\,(1-\rho)^{-\alpha-1}\,\|A_{\mathcal S\mathcal S'}\|.
\]
This proves the lemma.
\end{proof}

\begin{lemma}[Bound on cross--blocks of $M^k$]
\label{lem:Ak-SSp}
Let $M\in\mathbb{R}^{N\times N}$ with block decomposition
\[
M=
\begin{bmatrix}
M_{\mathcal S} & M_{\mathcal S\mathcal S'}\\
M_{\mathcal S'\mathcal S} & M_{\mathcal S'}
\end{bmatrix}.
\]
Then for all integers $k\ge 1$,
\[
\big\|[M^k]_{\mathcal S\mathcal S'}\big\|
\;\le\;
k\,\|M\|^{k-1}\,\|M_{\mathcal S\mathcal S'}\|.
\]
\end{lemma}

\begin{proof}
Write $F_k := [M^k]_{\mathcal S\mathcal S'}$. From block multiplication of
$M^{k+1}=M\cdot M^k$ we obtain the recursion
\[
F_{k+1}
=
M_{\mathcal S} F_k + M_{\mathcal S\mathcal S'}[M^k]_{\mathcal S'}.
\]
From submultiplicativity and subadditivity, we have
\[
\|F_{k+1}\|
\le
\|M_{\mathcal S}\|\|F_k\|
+
\|M_{\mathcal S\mathcal S'}\|\,\|[M^k]_{\mathcal S'}\|.
\]
Since $[M^k]_{\mathcal S'} = \Pi_{\mathcal S'} M^k \Pi_{\mathcal S'}^\top$
for a coordinate projector $\Pi_{\mathcal S'}$ with $\|\Pi_{\mathcal S'}\|\le 1$,
we have
$\|[M^k]_{\mathcal S'}\| \le \|M^k\| \le \|M\|^k$, and likewise
$\|M_{\mathcal S}\|\le \|M\|$.  Defining $f_k:=\|F_k\|$, this yields
\[
f_{k+1}
\le
\|M\|\,f_k + \|M_{\mathcal S\mathcal S'}\|\,\|M\|^k.
\]
Via induction with base case $f_1=\|M_{\mathcal S\mathcal S'}\|$ (or $f_0=0$), we obtain $f_k\le k\|M\|^{k-1}\|M_{\mathcal S\mathcal S'}\|$.
\end{proof}

\paragraph{Main Bound in Terms of $\overline{A}$, $\beta$ and $A_{\mathcal S\mathcal S'}$.}
Combining~\eqref{eq:delta-beta-C-bound} with Lemma~\ref{lem:spec-relation} gives
\begin{align}
\|\Delta_{-\beta}\|
&\le
\frac{\sin(\pi\beta)}{\pi}
B(1-\beta,2+\beta)\,
\frac{\|C_{\mathcal S\mathcal S'}\|^2}
{(1-\theta)\,\lambda_{\min}(C)^{2+\beta}}\\[0.3em]
&\le
\frac{\sin(\pi\beta)}{\pi}
B(1-\beta,2+\beta)\,
\frac{\sigma^4\alpha^2(1-\rho)^{-2\alpha-2}\|A_{\mathcal S\mathcal S'}\|^2}
{(1-\theta)\,\big(\sigma^2(1-\lambda_{\min}(\overline{A}))^{-\alpha}\big)^{2+\beta}}\\[0.3em]
&=
\frac{\sin(\pi\beta)}{\pi}
B(1-\beta,2+\beta)\,
\frac{\alpha^2(1-\lambda_{\min}(\overline{A}))^{\alpha(2+\beta)}}
{(1-\rho)^{2\alpha+2}}\,
\frac{1}{1-\theta}\,
\sigma^{-2\beta}
\|A_{\mathcal S\mathcal S'}\|^2.
\end{align}
Define the constant
\begin{equation}
H(\overline{A},\beta)
:=
\frac{\sin(\pi\beta)}{\pi}
\frac{\Gamma(1-\beta)\Gamma(2+\beta)(1-\lambda_{\min}(\overline{A}))^{(2/\beta+1)}}
{2\beta^2(1-\rho)^{2/\beta+2}},
\end{equation}
recalling that $\beta=1/\alpha$, so that
\begin{equation}
\|\Delta_{-\beta}\|
\;\le\;
\frac{H(\overline{A},\beta)}{\sigma^{2\beta}(1-\theta)}\,
\|A_{\mathcal S\mathcal S'}\|^2.
\label{eq:Delta-final-bound}
\end{equation}

We can now state the main theorem of this appendix section.

\begin{theorem}[Structural consistency of $(C_{\mathcal S})^{-1/\alpha}$]
\label{thm:frac-struct-consistency}
Let $C = \sigma^2 \left(I-\overline{A}\right)^{-\alpha}$ with $\overline{A}=\overline{A}^\top$ and $\rho(\overline{A})<1$, and let
$\beta=1/\alpha\in(0,1)$.  Define the stability parameter
\begin{equation}
s\left(\overline{A},\beta\right):= \frac{\beta\left(1-\rho\left(\overline{A}\right)\right)^{\left(1+1/\beta\right)}}{\left(1-\lambda_{\min}\left(\overline{A}\right)\right)^{1/\beta}}
\end{equation}
and the control parameter
\begin{equation}
c\left(\overline{A},\beta\right):= s\left(\overline{A},\beta\right) \sqrt{\frac{a_{\min}}{4 s\left(\overline{A},\beta\right)^2H\left(\overline{A},\beta\right)+a_{\min}}},
\end{equation}
where we recall that 
\begin{equation}
a_{\min}
:=
\min\{A_{ij}>0:\,i\neq j,\ i,j\in\mathcal S\}.
\end{equation}
If
\begin{equation}
\|A_{\mathcal S\mathcal S'}\| < g\left(\overline{A},\beta\right):=\min\left\{s\left(\overline{A},\beta\right),\,c\left(\overline{A},\beta\right)\right\}
\end{equation}
then $(C_{\mathcal S})^{-1/\alpha}$ is structurally consistent.
\end{theorem}

\begin{proof}
First, recall that we need $\theta<1$ to grant convergence of the Neumann series (stability). This is granted if
\begin{equation}
\theta = \frac{\|C_{\mathcal{S}\mathcal{S}'}\|^2}{\lambda_{\min}\left(C_{\mathcal{S}}\right)\lambda_{\min}\left(C_{\mathcal{S}'}\right)}\leq \alpha^2 \frac{\left(1-\lambda_{\min}\left(\overline{A}\right)\right)^{2\alpha}}{\left(1-\rho\left(\overline{A}\right)\right)^{2\alpha+2}}\|A_{\mathcal S \mathcal S'}\|^2<1,
\end{equation}
where the first inequality holds in view of Lemma~\ref{lem:spec-relation}. Thus, we have 
\begin{equation}
\|A_{\mathcal S \mathcal S'}\|< \frac{\left(1-\rho\left(\overline{A}\right)\right)^{\alpha+1}}{\alpha\left(1-\lambda_{\min}\left(\overline{A}\right)\right)^{\alpha}}=:s\left(\overline{A},\beta\right).
\label{eq:stab}
\end{equation}
We have structural consistency if
\begin{equation}
\operatorname{Osc}(\Delta_{-\beta})
< \frac{a_{\min}}{2\,\sigma^{2\beta}}.
\end{equation}
Recall from Appendix~\ref{sec:oscandall} that 
\begin{equation}
\operatorname{Osc}(\Delta_{-\beta})
\le 2\|\Delta_{-\beta}\|
< \frac{a_{\min}}{2\,\sigma^{2\beta}},
\end{equation}
which yields the sufficient condition
\begin{equation}
\|\Delta_{-\beta}\|
< \frac{a_{\min}}{4\,\sigma^{2\beta}}.
\end{equation}
We have that
\begin{equation}
\theta\leq \frac{\|A_{\mathcal S \mathcal S'}\|^2}{s\left(\overline{A},\beta\right)^2}<1,
\label{eq:theta}
\end{equation}
where the latter inequality comes from the stability condition~\eqref{eq:stab}.
From~\eqref{eq:Delta-final-bound}, we have
\begin{equation}
\|\Delta_{-\beta}\|
\;\le\;
\frac{H(\overline{A},\beta)}{\sigma^{2\beta}(1-\theta)}\,
\|A_{\mathcal S\mathcal S'}\|^2\leq \frac{H(\overline{A},\beta)}{\sigma^{2\beta}\left(1-\frac{\|A_{\mathcal S \mathcal S'}\|^2}{s\left(\overline{A},\beta\right)^2}\right)}\,
\|A_{\mathcal S\mathcal S'}\|^2
\end{equation}
where the second inequality holds in view of~\eqref{eq:theta}. Thus, to grant structural consistency, it is sufficient to have
\begin{equation}
\frac{H(\overline{A},\beta)}{\sigma^{2\beta}\left(1-\frac{\|A_{\mathcal S \mathcal S'}\|^2}{s\left(\overline{A},\beta\right)^2}\right)}\,
\|A_{\mathcal S\mathcal S'}\|^2< \frac{a_{\min}}{4\,\sigma^{2\beta}},
\end{equation}
which yields 
\begin{equation}
\|A_{\mathcal S \mathcal S'}\| < c\left(\overline{A},\beta\right).
\end{equation}
This concludes the proof.
\end{proof}

\paragraph{On the small cross--block condition.}
The structural consistency guarantee hinges on a spectral smallness
condition on the cross--block $A_{\mathcal S\mathcal S'}$.  At a high level,
this condition requires that latent nodes do not exert an excessively
strong \emph{coherent} influence on the observed subsystem.  This is a
natural requirement: if a small number of unobserved nodes act as very
strong confounders, then no method based solely on the observed
covariance can hope to disentangle direct from mediated interactions.
Our assumption is in the same spirit as standard identifiability
conditions in latent--variable graphical models and SEMs, which
control the strength or complexity of latent--to--observed couplings
(e.g., via operator norm, degree, or incoherence-type conditions).
In diffusive settings---such as those arising from Matérn--type
operators---the condition is particularly benign: when $\rho(A)$ is
bounded away from $1$ and the observed subgraph is reasonably well
connected, the constants in our bound become generous, and the required
upper bound on $\|A_{\mathcal S\mathcal S'}\|$ allows for many weak latent
connections.  In that regime, latent variables may still be numerous,
but their aggregate influence on the observed subsystem remains
``spread out'' rather than concentrated along a single dominant
direction, and structural recovery from $C_{\mathcal S}$ is possible.
Conversely, when the condition fails, this corresponds to pathological
regimes where latent confounding dominates observable dependencies and
structural recovery from marginal covariance is information--theoretically
impossible without further assumptions.

\paragraph{Example: inhomogeneous Erd\H{o}s--R\'enyi models.}
The smallness condition on $A_{\mathcal S\mathcal S'}$ is also satisfied
with high probability under simple random graph priors.  As a concrete
example, consider an inhomogeneous Erd\H{o}s--R\'enyi model on
$\mathcal S\cup\mathcal S'$ in which edges within $\mathcal S$, within
$\mathcal S'$, and across $(\mathcal S,\mathcal S')$ occur independently
with probabilities $p_{\mathrm{obs}}$, $p_{\mathrm{lat}}$, and
$p_{\mathrm{cross}}$, respectively, and edge weights are drawn from a
bounded distribution and optionally rescaled to enforce $\rho(A)<1$
(e.g., via a global scaling factor).
In this setting, $A_{\mathcal S\mathcal S'}$ is a rectangular random
matrix with independent bounded entries supported on the cross edges.
Standard random matrix bounds imply that its spectral norm concentrates
around
\[
\|A_{\mathcal S\mathcal S'}\|
=
O\!\Big(\sqrt{n_{\mathcal S}p_{\mathrm{cross}}}
+
\sqrt{n_{\mathcal S'}p_{\mathrm{cross}}}\Big)
\quad
\text{(up to the chosen weight scale)},
\]
while the intrinsic spectral scales of $A_{\mathcal S}$ and $A_{\mathcal S'}$
typically grow with $\sqrt{n_{\mathcal S}p_{\mathrm{obs}}}$ and
$\sqrt{n_{\mathcal S'}p_{\mathrm{lat}}}$, respectively (again, up to scaling).
Thus, whenever cross connectivity is sparser or more weakly weighted than
within--block connectivity (e.g., $p_{\mathrm{cross}}\ll p_{\mathrm{obs}}$,
or cross edges are scaled down), $\|A_{\mathcal S\mathcal S'}\|$ becomes
small relative to the intrinsic spectral scales with high probability,
and our consistency condition is met.  In other words, for a broad class
of random graph models that favor stronger within--system coupling than
latent--to--observed coupling, the structural consistency of
$(C_{\mathcal S})^{-1/\alpha}$ is not only plausible but holds generically
in the large--$n$ limit.

\subsection{Structural Consistency for Fractional Fields: General Exponents $\beta \in \mathbb{R}\setminus\left\{0\right\}$}
\label{app:frac-consistency-dt}

In this section, we establish structural consistency for fractional powers of
the observed covariance under partial observability, for \emph{arbitrary real
exponents} $\beta\in\mathbb R\setminus\left\{0\right\}$.  The result extends the $\beta \in (0,1)$ case discussed in the previous appendix subsection by considering the generalized integral representation over the complex plane (Dunford-Taylor) of $C^{-\beta}$ and $\left(C_{\mathcal{S}}\right)^{-\beta}$.

\paragraph{Dunford--Taylor Integral Representation.}
Let $f\,:\,\mathbb{C}\rightarrow \mathbb{C}$ be defined by $f(z)=z^{-\beta}$. Since $C\in\mathbb{S}_{++}^{N}$ is symmetric positive definite, then $\operatorname{spec}(C)\subset(0,\infty)$, where 
\begin{equation}
\operatorname{spec}(C):= \left\{\lambda_1\left(C\right),\lambda_2(C),\ldots,\lambda_N(C)\right\},
\end{equation}
is the spectrum of $C$ (collecting its eigenvalues). Let $\Gamma$ be a positively oriented, simple, closed contour over the complex plane $\mathbb{C}$ obeying the following properties
\begin{enumerate}
\item $\operatorname{spec}(C)\subset\mathrm{int}(\Gamma)$,
\item $\Gamma\subset\mathbb C\setminus(-\infty,0]$,
\item $0\notin\mathrm{int}(\Gamma)$ if $\beta>0$.
\end{enumerate}
Then $f$ is holomorphic on and \emph{inside} $\Gamma$, and the Dunford--Taylor integral representation yields~\cite{Kato1995}
\begin{equation}
C^{-\beta}
=
\frac{1}{2\pi i}
\oint_\Gamma z^{-\beta}(zI-C)^{-1}\,dz,
\qquad
(C_{\mathcal S})^{-\beta}
=
\frac{1}{2\pi i}
\oint_\Gamma z^{-\beta}(zI-C_{\mathcal S})^{-1}\,dz.
\label{eq:DT}
\end{equation}

Subtracting the principal blocks gives
\begin{equation}
\Delta_{-\beta}
= (C_{\mathcal S})^{-\beta} - \left[C^{-\beta}\right]_{\mathcal{S}} =
\frac{1}{2\pi i}
\oint_\Gamma z^{-\beta}R_{\mathcal S}(z)\,dz,
\label{eq:DT-gap}
\end{equation}
where
\[
R_{\mathcal S}(z)
:=
(zI-C_{\mathcal S})^{-1}
-
[(zI-C)^{-1}]_{\mathcal S}.
\]

By taking the operator $L_2$ norm, we have the following bound on $\Delta_{-\beta}$
\begin{equation}
\|\Delta_{-\beta}\|
\le
\frac{1}{2\pi}
\int_\Gamma
|z^{-\beta}|\,
\|R_{\mathcal S}(z)\|\,
|dz|.
\label{eq:DT-basic-bound}
\end{equation}

Similarly to as done in the previous subsection, we will bound $\Delta_{-\beta}$ via bounding $R_{\mathcal S}$ (which is more amenable via Schur complement since it is built from a fixed inverse power).

\paragraph{Schur complement and resolvent bound.}
For $z\in\Gamma$, write
\[
zI-C=
\begin{bmatrix}
B(z) & -C_{\mathcal{S} \mathcal{S}'}\\
-C_{\mathcal{S}' \mathcal{S}} & D(z)
\end{bmatrix},
\qquad
B(z):=zI-C_{\mathcal S},\;
D(z):=zI-C_{\mathcal S'}.\;
\]
Schur complement gives
\begin{equation}
[(zI-C)^{-1}]_{\mathcal S}
=
(B(z)-H(z))^{-1},
\qquad
H(z):=C_{\mathcal S\mathcal S'} D(z)^{-1}C_{\mathcal S'\mathcal S}.
\label{eq:Schur_Hz}
\end{equation}
Hence
\begin{equation}
R_{\mathcal S}(z)
=
B(z)^{-1}-(B(z)-H(z))^{-1}
=
-(B(z)-H(z))^{-1}H(z)B(z)^{-1}.
\label{eq:Rz}
\end{equation}

Define the \emph{contour} quantities
\begin{equation}
\Theta_\Gamma
:=
\sup_{z\in\Gamma}
\|B(z)^{-1}H(z)\|,
\label{eq:ThetaGamma}
\end{equation}
and
\begin{equation}
\mathcal K_\Gamma(\beta,C)
:=
\frac{1}{2\pi}
\int_\Gamma
|z^{-\beta}|\,
\|B(z)^{-1}\|^2\,
\|D(z)^{-1}\|\,
|dz|.
\label{eq:KGamma}
\end{equation}

The next lemma offers a bound to the $L_2$-norm of $R_{\mathcal{S}}$ (which is relevant to bound the gap $\Delta_{-\beta}$ in view of~\eqref{eq:DT-basic-bound}).

\begin{lemma}[Resolvent bound]
If $\Theta_\Gamma<1$, then for all $z\in\Gamma$,
\begin{equation}
\|R_{\mathcal S}(z)\|
\le
\frac{\|C_{\mathcal S\mathcal S'}\|^2}{1-\Theta_\Gamma}
\|B(z)^{-1}\|^2\|D(z)^{-1}\|.
\label{eq:Rz_bound}
\end{equation}
\end{lemma}

\begin{proof}
If $\Theta_\Gamma<1$, then $(I-B(z)^{-1}H(z))$ is invertible for all $z\in\Gamma$ and
\[
(B(z)-H(z))^{-1}
=
(I-B(z)^{-1}H(z))^{-1}B(z)^{-1}
=
\sum_{m\ge0}(B(z)^{-1}H(z))^mB(z)^{-1}.
\]
Hence
\[
\|(B(z)-H(z))^{-1}\|
\le
\frac{\|B(z)^{-1}\|}{1-\Theta_\Gamma}.
\]
In view of the definition of $H(z)$ in~\eqref{eq:Schur_Hz}, since $C_{\mathcal{S}\mathcal{S}'}^{\top}=C_{\mathcal{S}'\mathcal{S}}$, and from the submultiplicativity of $\|\cdot\|$, we have that $\|H(z)\|\le\|C_{\mathcal S\mathcal S'}\|^2\|D(z)^{-1}\|$. This concludes the proof.
\end{proof}

Substituting the bound~\eqref{eq:Rz_bound} into~\eqref{eq:DT-basic-bound} gives the perturbation estimate for the $L_2$ norm of the gap $\Delta_{-\beta}$
\begin{equation}
\boxed{
\|\Delta_{-\beta}\|
\le
\frac{\mathcal K_\Gamma(\beta,C)}{1-\Theta_\Gamma}
\|C_{\mathcal S\mathcal S'}\|^2.
}
\label{eq:DT-main-bound}
\end{equation}

\paragraph{Choice of the contour.}
The terms $\mathcal K_\Gamma(\beta,C)$ and
$\Theta_\Gamma$ characterizing the bounds for $\Delta_{-\beta}$ depend on the choice of the curve $\Gamma$. We now characterize more concretely this dependence so to latter have a more explicit bound for $\Delta_{-\beta}$ and hence for structural consistency. 

\begin{theorem}[Structural consistency of $(C_{\mathcal S})^{-\beta}$ for $\beta\in\mathbb{R}\setminus\left\{0\right\}$ via Dunford--Taylor for a particular $\Gamma$.]
\label{thm:DT-struct-consistency}
Let $C\succ 0$. Assume there exists a positively oriented, simple closed contour $\Gamma\subset\mathbb{C}$
such that:
\begin{enumerate}
\item $\mathrm{spec}(C)\subset \mathrm{int}(\Gamma)$ and $0\notin \mathrm{int}(\Gamma)$, so that the principal-branch function $z\mapsto z^{-\beta}$ is analytic on and inside~$\Gamma$.
\item The following resolvent separation quantities are positive:
\[
\inf_{z\in\Gamma}{\sf dist}\!\big(z,\mathrm{spec}(C_{\mathcal S})\big)>0,\quad
\inf_{z\in\Gamma}{\sf dist}\!\big(z,\mathrm{spec}(C_{\mathcal S'})\big)>0.
\]
\end{enumerate}
Define the contour-dependent stability parameter
\begin{equation}
\Theta_\Gamma
:=
\sup_{z\in\Gamma}
\frac{\|C_{\mathcal S\mathcal S'}\|^2}
{{\sf dist}(z,\mathrm{spec}(C_{\mathcal S}))\,{\sf dist}(z,\mathrm{spec}(C_{\mathcal S'}))},
\label{eq:ThetaGamma-def}
\end{equation}
and assume $\Theta_\Gamma<1$.

Define the contour constant
\begin{equation}
\mathcal K_\Gamma(\beta,C)
:=
\frac{1}{2\pi}
\oint_{\Gamma}
|z^{-\beta}|\,
\frac{|dz|}{{\sf dist}(z,\mathrm{spec}(C_{\mathcal S}))^2\,{\sf dist}(z,\mathrm{spec}(C_{\mathcal S'}))}.
\label{eq:KGamma-def}
\end{equation}
If
\begin{equation}
\boxed{
\frac{\mathcal K_\Gamma(\beta,C)}{1-\Theta_\Gamma}
\|C_{\mathcal S\mathcal S'}\|^2
\;<\;
\frac{a_{\min}}{4\sigma^{2\beta}},}
\label{eq:DT-consistency-condition}
\end{equation}
then $(C_{\mathcal S})^{-\beta}$ is structurally consistent.
\end{theorem}

\begin{proof}
Using $(B-H)^{-1}=(I-B^{-1}H)^{-1}B^{-1}$ and the Neumann bound
$\|(I-B^{-1}H)^{-1}\|\le (1-\|B^{-1}H\|)^{-1}$ (valid when $\|B^{-1}H\|<1$), we obtain
\begin{align}
\|R_{\mathcal S}(z)\|
&\le
\frac{\|B(z)^{-1}\|^2\,\|H(z)\|}{1-\|B(z)^{-1}H(z)\|}.
\label{eq:R-bound-pre}
\end{align}
Now,
\[
\|H(z)\|\le \|C_{\mathcal S \mathcal S'}\|^2\,\|D(z)^{-1}\|,
\qquad
\|B(z)^{-1}H(z)\|
\le \|B(z)^{-1}\|\,\|C_{\mathcal S \mathcal S'}\|^2\,\|D(z)^{-1}\|.
\]
Since $\| (zI-C_{\mathcal S})^{-1}\|=1/{\sf dist}(z,\mathrm{spec}(C_{\mathcal S}))$ (and similarly for $C_{\mathcal S'}$),
we have
\[
\|B(z)^{-1}\|=\frac{1}{{\sf dist}(z,\mathrm{spec}(C_{\mathcal S}))},
\qquad
\|D(z)^{-1}\|=\frac{1}{{\sf dist}(z,\mathrm{spec}(C_{\mathcal S'}))}.
\]
Thus,
\[
\|B(z)^{-1}H(z)\|
\le
\frac{\|C_{\mathcal S \mathcal S'}\|^2}{{\sf dist}(z,\mathrm{spec}(C_{\mathcal S}))\,{\sf dist}(z,\mathrm{spec}(C_{\mathcal S'}))}.
\]
Taking the supremum over $z\in\Gamma$ gives exactly $\Theta_\Gamma$ in \eqref{eq:ThetaGamma-def},
and the assumption $\Theta_\Gamma<1$ implies the Neumann bound holds uniformly on~$\Gamma$.
Plugging these estimates into \eqref{eq:R-bound-pre} yields, for all $z\in\Gamma$,
\[
\|R_{\mathcal S}(z)\|
\le
\frac{\|C_{\mathcal S \mathcal S'}\|^2}{1-\Theta_\Gamma}\,
\frac{1}{{\sf dist}(z,\mathrm{spec}(C_{\mathcal S}))^2\,{\sf dist}(z,\mathrm{spec}(C_{\mathcal S'}))}.
\]
Therefore,
\[
\|\Delta_{-\beta}\|
\le
\frac{\|C_{\mathcal S \mathcal S'}\|^2}{1-\Theta_\Gamma}\,
\frac{1}{2\pi}\oint_\Gamma
|z^{-\beta}|\,
\frac{|dz|}{{\sf dist}(z,\mathrm{spec}(C_{\mathcal S}))^2\,{\sf dist}(z,\mathrm{spec}(C_{\mathcal S'}))}.
\]
Recognizing $\mathcal K_\Gamma(\beta,C)$ from \eqref{eq:KGamma-def}, inequality~\eqref{eq:DT-consistency-condition} implies structural consistency.
\end{proof}

Now, since $C\succ0$ and $\operatorname{spec}(C)\subset[\lambda_{\min},\lambda_{\max}]$,
a possible particular choice for the curve $\Gamma$ is a \emph{circle} enclosing the spectrum of $C$, avoiding the negative real
axis and the origin. This offers a more explicit bound for $\mathcal K_\Gamma(\beta,C)$ and
$\Theta_\Gamma$ in terms of $\lambda_{\min}$,
$\lambda_{\max}$, $|\beta|$, and the contour radius.

\begin{corollary}[Circular contour]
\label{cor:K-circle-explicit}
Assume $\mathrm{spec}(C)\subset[m,M]\subset(0,\infty)$.
Fix $\varepsilon\in(0,m)$ and define the circle
\[
\Gamma_\varepsilon := \{z:\ |z-c|=R\},
\qquad
c:=\frac{m+M}{2},
\quad
R:=\frac{M-m}{2}+\varepsilon.
\]
Then for all $z\in\Gamma_\varepsilon$,
\[
{\sf dist}(z,\mathrm{spec}(C_{\mathcal S}))\ge \varepsilon,
\qquad
{\sf dist}(z,\mathrm{spec}(C_{\mathcal S'}))\ge \varepsilon,
\qquad
\min_{z\in\Gamma_\varepsilon}|z|=m-\varepsilon,
\quad
\max_{z\in\Gamma_\varepsilon}|z|=M+\varepsilon.
\]
Hence
\begin{equation}
\boxed{
\Theta_{\Gamma_\varepsilon}\;\le\;\frac{\|C_{\mathcal S\mathcal S'}\|^2}{\varepsilon^2}.}
\label{eq:Theta-circle}
\end{equation}
Moreover, the contour constant admits the explicit bound
\begin{equation}
\boxed{
\mathcal K_{\Gamma_\varepsilon}(\beta,C)
\;\le\;
R\,\varepsilon^{-3}\times
\begin{cases}
(m-\varepsilon)^{-\beta}, & \beta\ge 0,\\[0.2em]
(M+\varepsilon)^{-\beta}, & \beta<0.
\end{cases}}
\label{eq:K-circle}
\end{equation}
Therefore, if $\Theta_{\Gamma_\varepsilon}<1$ and
\[
\frac{\mathcal K_{\Gamma_\varepsilon}(\beta,C)}{1-\Theta_{\Gamma_\varepsilon}}\,
\|C_{\mathcal S\mathcal S'}\|^2
<\frac{a_{\min}}{4\sigma^{2\beta}},
\]
then $(C_{\mathcal S})^{-\beta}$ is structurally consistent.
\end{corollary}

\begin{proof}
Since $\mathrm{spec}(C)\subset[m,M]$ and the eigenvalues of any principal submatrix interlace, then
$\mathrm{spec}(C_{\mathcal S})\subset[m,M]$ and $\mathrm{spec}(C_{\mathcal S'})\subset[m,M]$.
By construction of $\Gamma_\varepsilon$, the distance from the circle to the interval $[m,M]$ is at least $\varepsilon$,
hence ${\sf dist}(z,\mathrm{spec}(C_{\mathcal S}))\ge\varepsilon$ and ${\sf dist}(z,\mathrm{spec}(C_{\mathcal S'}))\ge\varepsilon$ for all $z\in\Gamma_\varepsilon$.
This implies \eqref{eq:Theta-circle}.

Next, the definition \eqref{eq:KGamma-def} and the bounds above give
\begin{equation}
\mathcal K_{\Gamma_\varepsilon}(\beta,C)
\le
\frac{1}{2\pi}\oint_{\Gamma_\varepsilon}
\sup_{z\in\Gamma_\varepsilon}|z^{-\beta}|\,
\varepsilon^{-3}\,|dz|
=\frac{|\Gamma_\varepsilon|}{2\pi}\,\varepsilon^{-3}\,\sup_{z\in\Gamma_\varepsilon}|z^{-\beta}|
=
R\,\varepsilon^{-3}\,\sup_{z\in\Gamma_\varepsilon}|z^{-\beta}|.
\end{equation}
Finally, since $|z|$ ranges between $m-\varepsilon$ and $M+\varepsilon$ on the circle,
$\sup |z^{-\beta}|$ equals $(m-\varepsilon)^{-\beta}$ if $\beta\ge 0$ and $(M+\varepsilon)^{-\beta}$ if $\beta<0$,
proving \eqref{eq:K-circle}.
\end{proof}

From the Matérn model, we have $C=\sigma^2(I-\overline{A})^{-\alpha}$ with $\overline{A}=\overline{A}^\top$, $\rho(\overline{A})<1$ and $\overline{A}=I-\kappa^2 D-L$. Namely, recall that the off-diagonal entries of $A$ coincide with the ones of $\overline{A}$.
Then, $\mathrm{spec}(C)\subset[m,M]$ with
\begin{equation}
m=\sigma^2(1-\rho(\overline{A}))^{-\alpha},
\qquad
M=\sigma^2(1-\lambda_{\min}(\overline{A}))^{-\alpha}.
\label{eq:mM}
\end{equation}

Further, define the covariance-to-graph cross--block amplification factor
\begin{equation}
L(\overline{A})
:=
\sigma^2\,\alpha\,(1-\rho(\overline{A}))^{-\alpha-1},
\qquad\text{so that}\qquad
\|C_{\mathcal S \mathcal S'}\|\le L(\overline{A})\,\|A_{\mathcal S\mathcal S'}\|.
\label{eq:L-def}
\end{equation}
where the latter inequality holds for $\alpha>0$. Hence the sufficient condition in Corollary~\ref{cor:K-circle-explicit} (with the circle $\Gamma_\varepsilon$)
can be expressed explicitly in terms of $\overline{A}$ and $\|A_{\mathcal S \mathcal S'}\|$ by substituting the above identities~\eqref{eq:mM} and the bound \eqref{eq:L-def} into \eqref{eq:K-circle} and~\eqref{eq:Theta-circle}, respectively.

\begin{theorem}[Explicit small cross--block condition in terms of $\overline{A}$ and $\beta>0$ (circular contour)]
\label{cor:DT-minsc}
Define also the contour-dependent magnitude factor
\begin{equation}
Z_{\varepsilon}(\beta;\overline{A})
:=
\sup_{z\in\Gamma_\varepsilon}|z^{-\beta}|
=
\begin{cases}
(m-\varepsilon)^{-\beta}, & \beta\ge 0,\\[0.2em]
(M+\varepsilon)^{-\beta}, & \beta<0,
\end{cases}
\label{eq:Z-eps}
\end{equation}
and set
\begin{equation}
K_{\varepsilon}(\beta;\overline{A})
:=
R\,\varepsilon^{-3}\,Z_{\varepsilon}(\beta;\overline{A}).
\label{eq:Keps-def}
\end{equation}

Now, define the \emph{stability} and \emph{control} parameters
\begin{align}
s_\varepsilon(\overline{A})
&:=
\frac{\varepsilon}{L(\overline{A})},
\label{eq:s-eps}
\\[0.4em]
c_\varepsilon(\overline{A},\beta)
&:=
s_\varepsilon(\overline{A})
\sqrt{
\frac{a_{\min}}{4\varepsilon^2 \sigma^{2\beta} K_{\varepsilon}(\beta;\overline{A}) + a_{\min}}
}.
\label{eq:c-eps}
\end{align}
If
\begin{equation}
\boxed{
\|A_{\mathcal S\mathcal S'}\|
<
g_{\varepsilon}(\overline{A},\beta)
:=
\min\left\{s_\varepsilon(\overline{A},\beta),\ c_\varepsilon(\overline{A},\beta)\right\},
}
\label{eq:g-eps}
\end{equation}
then $(C_{\mathcal S})^{-\beta}$ is structurally consistent.
\end{theorem}

\begin{proof}
From Corollary~\ref{cor:K-circle-explicit}, for the circle $\Gamma_\varepsilon$ we have
\[
\Theta_{\Gamma_\varepsilon}\le \frac{\|C_{\mathcal S \mathcal S'}\|^2}{\varepsilon^2},
\qquad
\mathcal K_{\Gamma_\varepsilon}(\beta,C)\le R\,\varepsilon^{-3}\,\sup_{z\in\Gamma_\varepsilon}|z^{-\beta}|
=K_{\varepsilon}(\beta;\overline{A}).
\]
Using $\|C_{\mathcal S \mathcal S'}\|\le L(\overline{A})\|A_{\mathcal S\mathcal S'}\|$ from~\eqref{eq:L-def}, the stability condition
$\Theta_{\Gamma_\varepsilon}<1$ is ensured by
\[
\frac{L(\overline{A})^2\|A_{\mathcal S\mathcal S'}\|^2}{\varepsilon^2}<1
\quad\Longleftrightarrow\quad
\|A_{\mathcal S\mathcal S'}\|<\frac{\varepsilon}{L(\overline{A})}
=:s_\varepsilon(\overline{A}, \beta).
\]
Under this condition,
Theorem~\ref{thm:DT-struct-consistency} (in its contour form \eqref{eq:DT-consistency-condition}) yields the sufficient requirement for structural consistency
\[
\frac{\mathcal K_{\Gamma_\varepsilon}(\beta,C)}{1-\Theta_{\Gamma_\varepsilon}}\,
\|C_{\mathcal S \mathcal S'}\|^2
<\frac{a_{\min}}{4\sigma^{2\beta}}.
\]
Using the bounds above, it suffices that
\[
\frac{K_{\varepsilon}(\beta;\overline{A})}{1-\frac{\|C_{\mathcal S\mathcal S'}\|^2}{\varepsilon^2}}\,\|C_{\mathcal S \mathcal S'}\|^2
<
\frac{a_{\min}}{4\sigma^{2\beta}}.
\]
Temporarily, let $y:=\|C_{\mathcal S \mathcal S'}\|^2$. Rearranging gives
\[
K_{\varepsilon}(\beta;\overline{A})\,y
<
\frac{a_{\min}}{4\sigma^{2\beta}}\Big(1-\frac{y}{\varepsilon^2}\Big)
\quad\Longleftrightarrow\quad
y\Big(K_{\varepsilon}(\beta;\overline{A})+\frac{a_{\min}}{4\varepsilon^{2}\sigma^{2\beta}}\Big)
<
\frac{a_{\min}}{4\sigma^{2\beta}}.
\]
Hence
\[
\|C_{\mathcal{S} \mathcal{S}'}\|^2
<
\frac{a_{\min}/4\sigma^{2\beta}}{K_{\varepsilon}(\beta;\overline{A})+\frac{a_{\min}}{4\varepsilon^{2}\sigma^{2\beta}}}.
\]
Finally, since $\|C_{\mathcal S \mathcal S'}\|\le L(\overline{A})\|A_{\mathcal S\mathcal S'}\|$, we obtain the sufficient condition
\[
\|A_{\mathcal S\mathcal S'}\|
<
\frac{\varepsilon}{L(\overline{A})}
\sqrt{
\frac{a_{\min}}{4 \varepsilon^2 \sigma^{2\beta} K_{\varepsilon}(\beta;\overline{A}) + a_{\min}}
}
=
c_\varepsilon(\overline{A},\beta),
\]
and combining with the stability bound yields \eqref{eq:g-eps}.
\end{proof}

\paragraph{On the small cross--block condition in the general--$\beta$ setting.}
The structural consistency guarantees obtained above rely on a spectral
smallness condition on the cross--block interaction matrix
$A_{\mathcal S\mathcal S'}$, which appears through the factor
$\|A_{\mathcal S\mathcal S'}\|^2$ in the perturbation bound.  In the
general Dunford--Taylor formulation, all dependence on the exponent
$\beta\in\mathbb{R}$ and on the spectral geometry of the observed
covariance $C_{\mathcal S}$ is captured by the contour--dependent constant
$\mathcal K_\Gamma(\beta,C)$ and the stability factor $(1-\Theta_\Gamma)^{-1}$.
Importantly, these quantities depend only on the spectrum of the observed
subsystem and on the choice of contour $\Gamma$, and are completely
independent of the latent--to--observed coupling.

At a conceptual level, the condition
\[
\mathcal K_\Gamma(\beta,C)\,
\frac{\|A_{\mathcal S\mathcal S'}\|^2}{1-\Theta_\Gamma}
\;\ll\;
a_{\min}
\]
requires that latent nodes do not exert an excessively strong
\emph{coherent} influence on the observed subsystem.  This requirement is
intrinsic rather than technical: if a small number of unobserved nodes
induces a near--rank--one or otherwise dominant perturbation of the
observed covariance, then no estimator based solely on marginal
second--order statistics can reliably disentangle direct interactions
from mediated ones.  The quadratic dependence on
$\|A_{\mathcal S\mathcal S'}\|$ reflects the fact that latent effects enter
the observed precision only through second--order Schur complement terms.

Crucially, the general--$\beta$ extension does not strengthen this
assumption relative to the fractional case $\beta\in(0,1)$.  Rather, it
makes explicit how the allowable magnitude of latent coupling trades off
with the spectral scale of $C$ and the chosen analytic contour through
$\mathcal K_\Gamma(\beta,C)$.  In diffusive regimes---such as those arising
from Matérn--type operators with $\rho(A)$ bounded away from $1$---these
constants remain moderate, and the condition allows for many weak latent
connections whose aggregate influence is ``spread out'' rather than
concentrated along a single dominant direction.

Conversely, when the condition fails, this corresponds precisely to
pathological regimes in which latent confounding dominates observable
dependencies, and identifiability from the marginal covariance
$C_{\mathcal S}$ becomes information--theoretically impossible.  In this
sense, the small cross--block assumption delineates the sharp boundary
between regimes where structural recovery is feasible and those where it
is fundamentally obstructed by latent confounding.

\section{Additional Experiments}

The core architecture for seizure detection is detailed in Fig.~\ref{fig:ML_model} (a similar CNN-based design was employed for churn prediction).
\begin{figure}[t]
	\centering
	\includegraphics[scale=0.45]{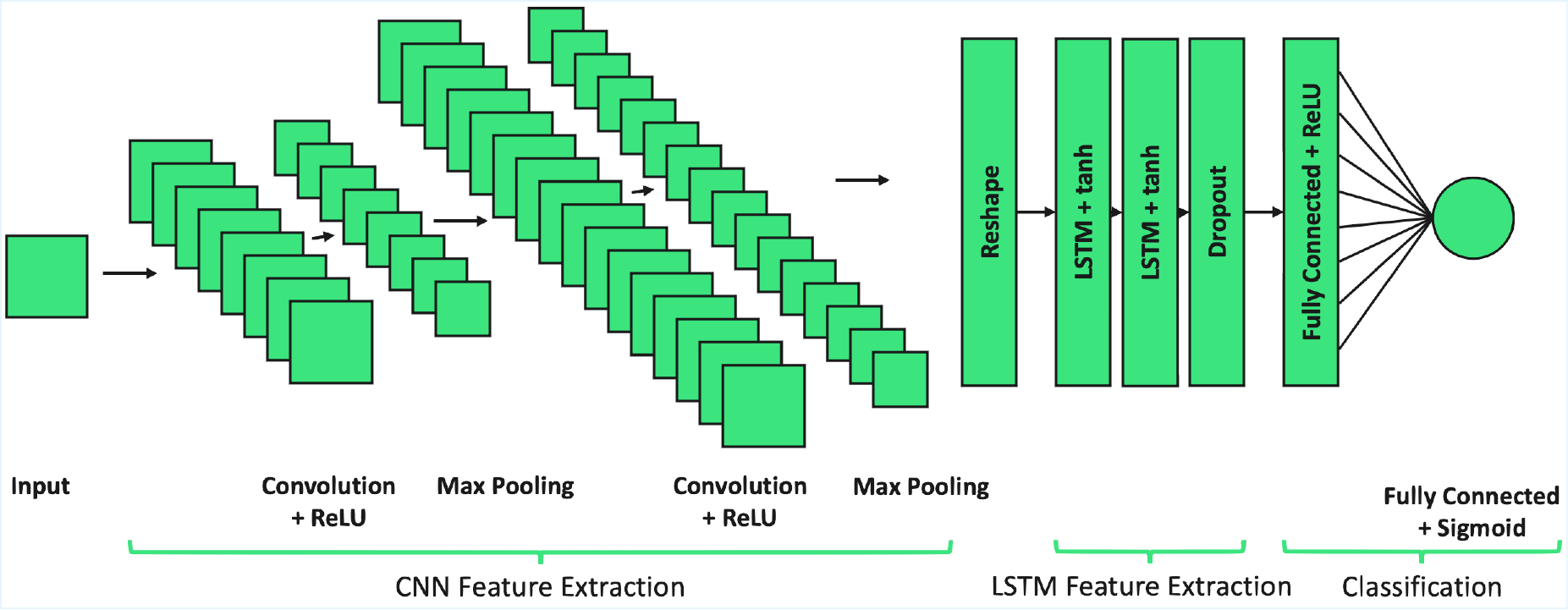}
	\caption{The main architecture ML model used.}
	\label{fig:ML_model}
\end{figure}

In contrast, Fig.~\ref{fig:features_across_pat} demonstrates that the CNN+LSTM hybrid achieves consistently high performance across all patients when utilizing our engineered structure-informed features. To further isolate the impact of these features, we evaluated standalone versions of the CNN (Fig.~\ref{fig:features_across_pat2}) and LSTM (Fig.~\ref{fig:features_across_pat3}). In both cases, the structure-informed features ensure stable, high-accuracy results regardless of the specific architecture. These findings underscore the significant performance boost provided by our feature engineering over raw data.
\begin{figure}[t]
	\centering
	\includegraphics[scale=0.7]{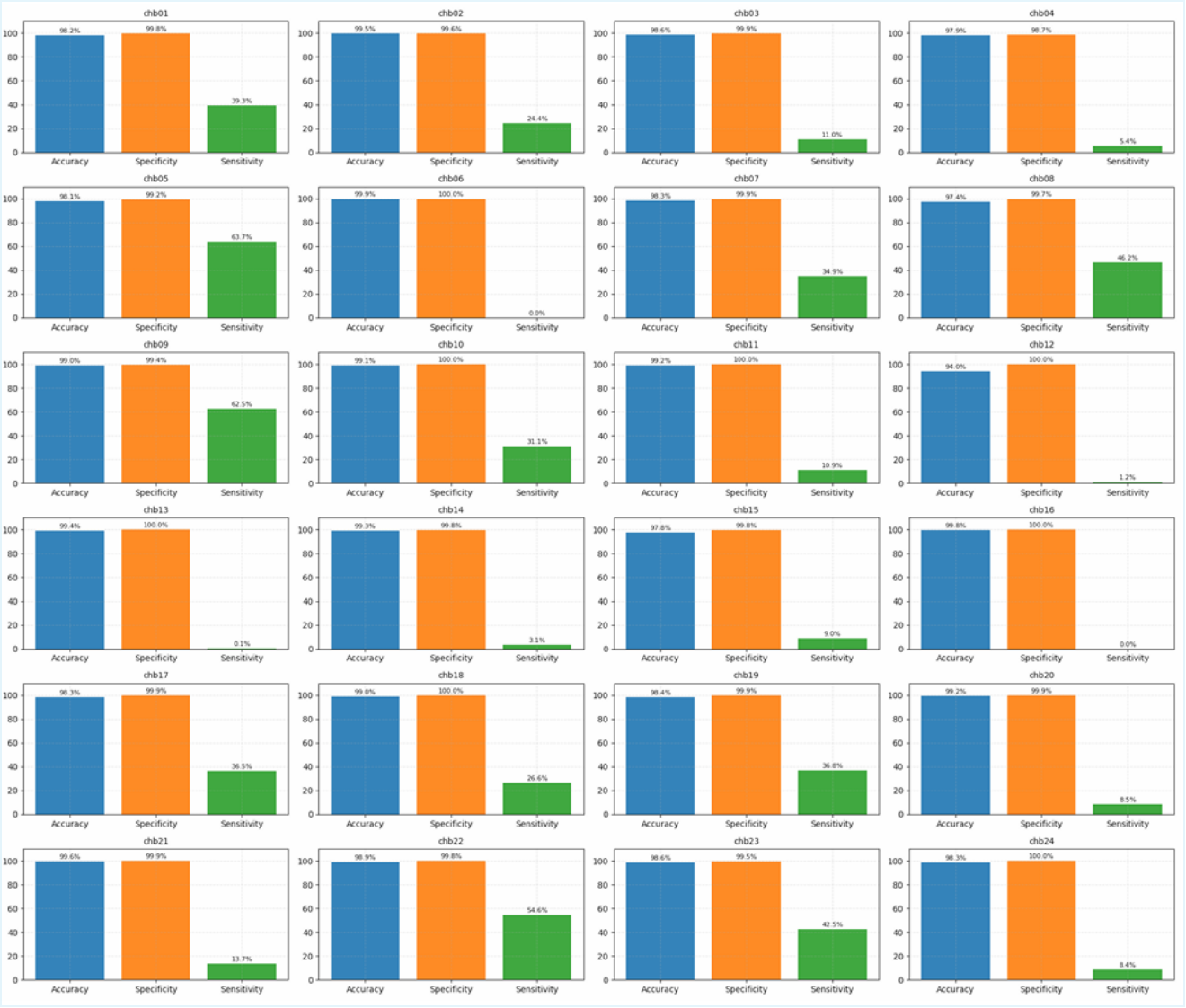}
	\caption{Performance of the core architecture (CNN+LSTM) trained on raw data for distinct patients.}
	\label{fig:raw_across_pat}
\end{figure}

\begin{figure}[t]
	\centering
	\includegraphics[scale=0.6]{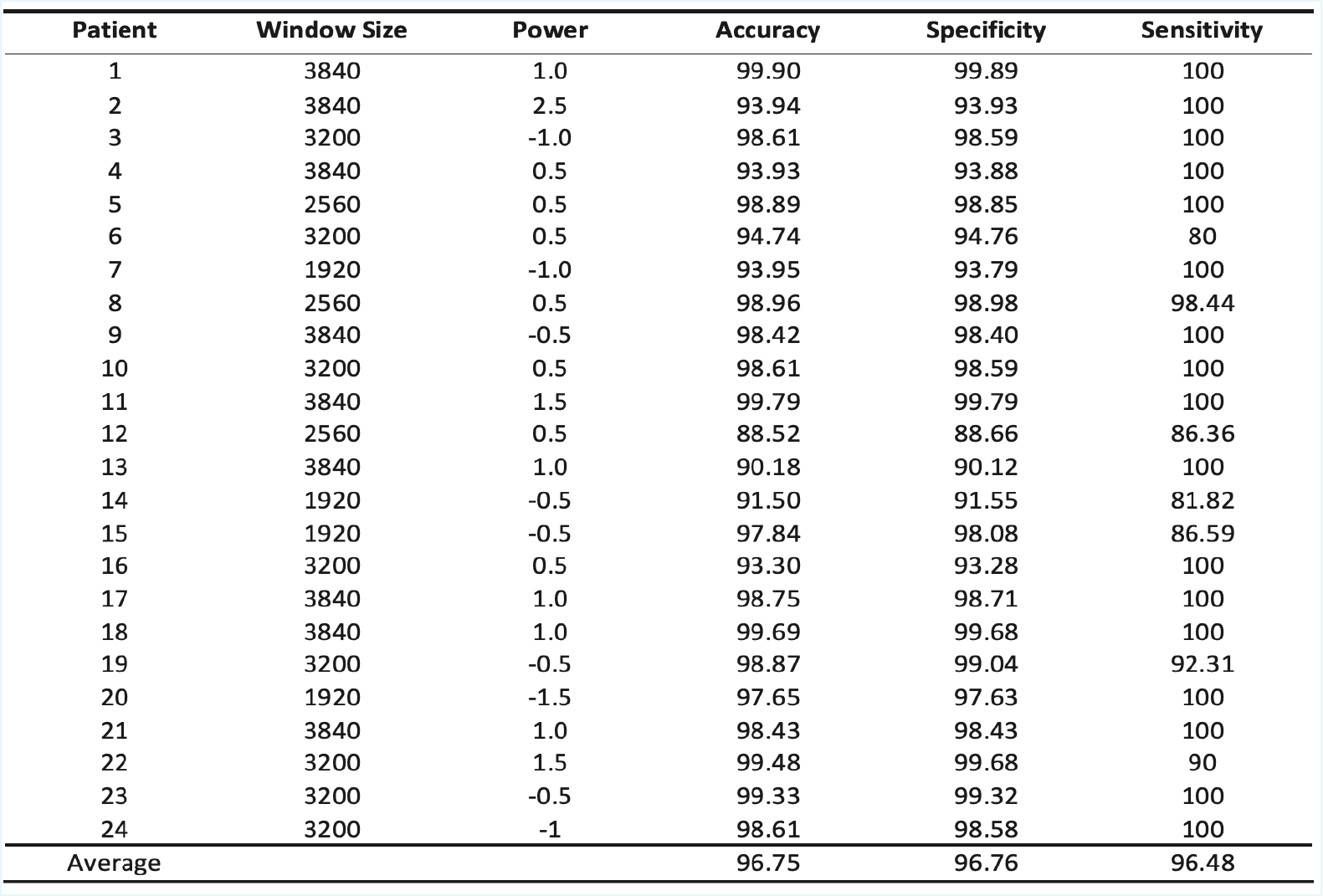}
	\caption{Performance of the core architecture (CNN+LSTM) trained with our features across distinct patients.}
	\label{fig:features_across_pat}
\end{figure}

\begin{figure}[t]
	\centering
	\includegraphics[scale=0.6]{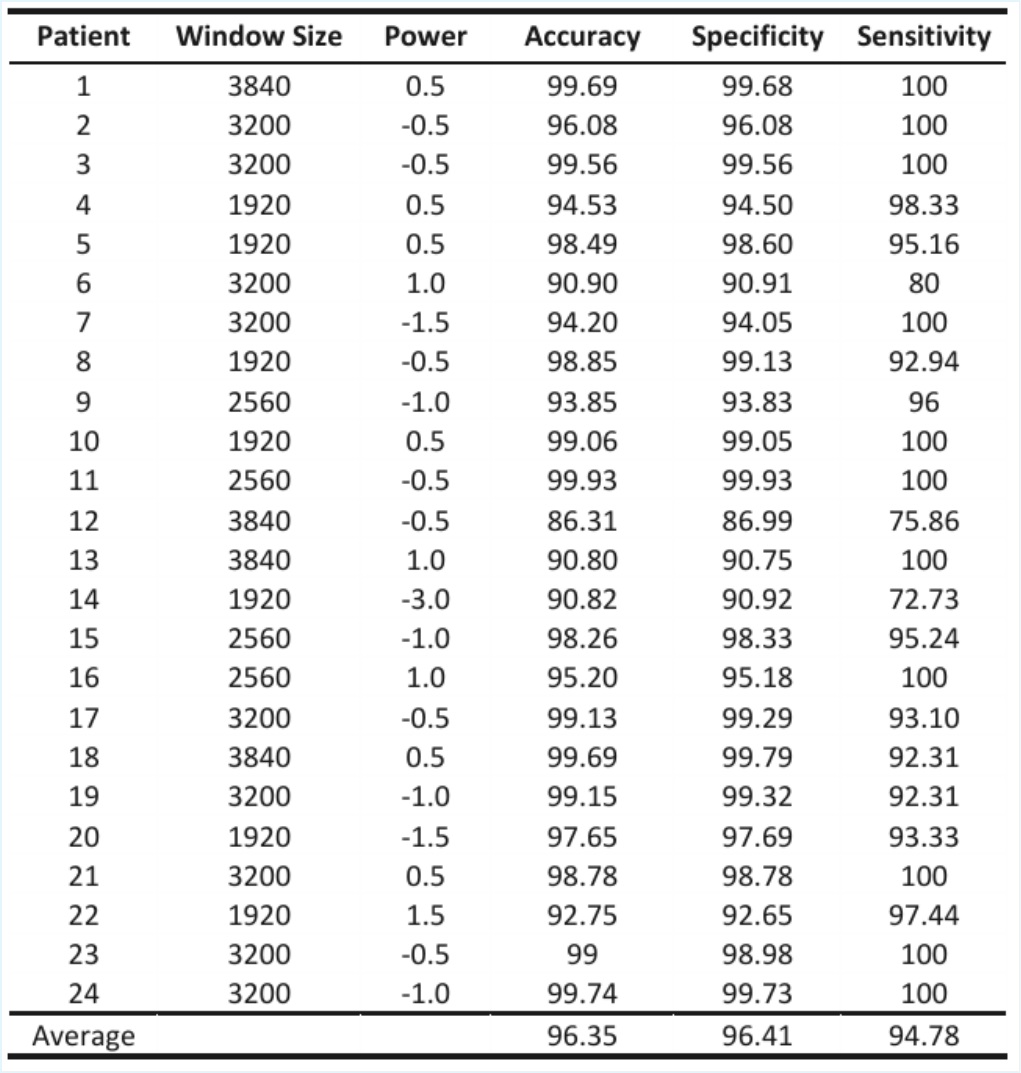}
	\caption{Performance for a stand alone CNN trained with our features across distinct patients.}
	\label{fig:features_across_pat2}
\end{figure}

\begin{figure}[t]
	\centering
	\includegraphics[scale=0.6]{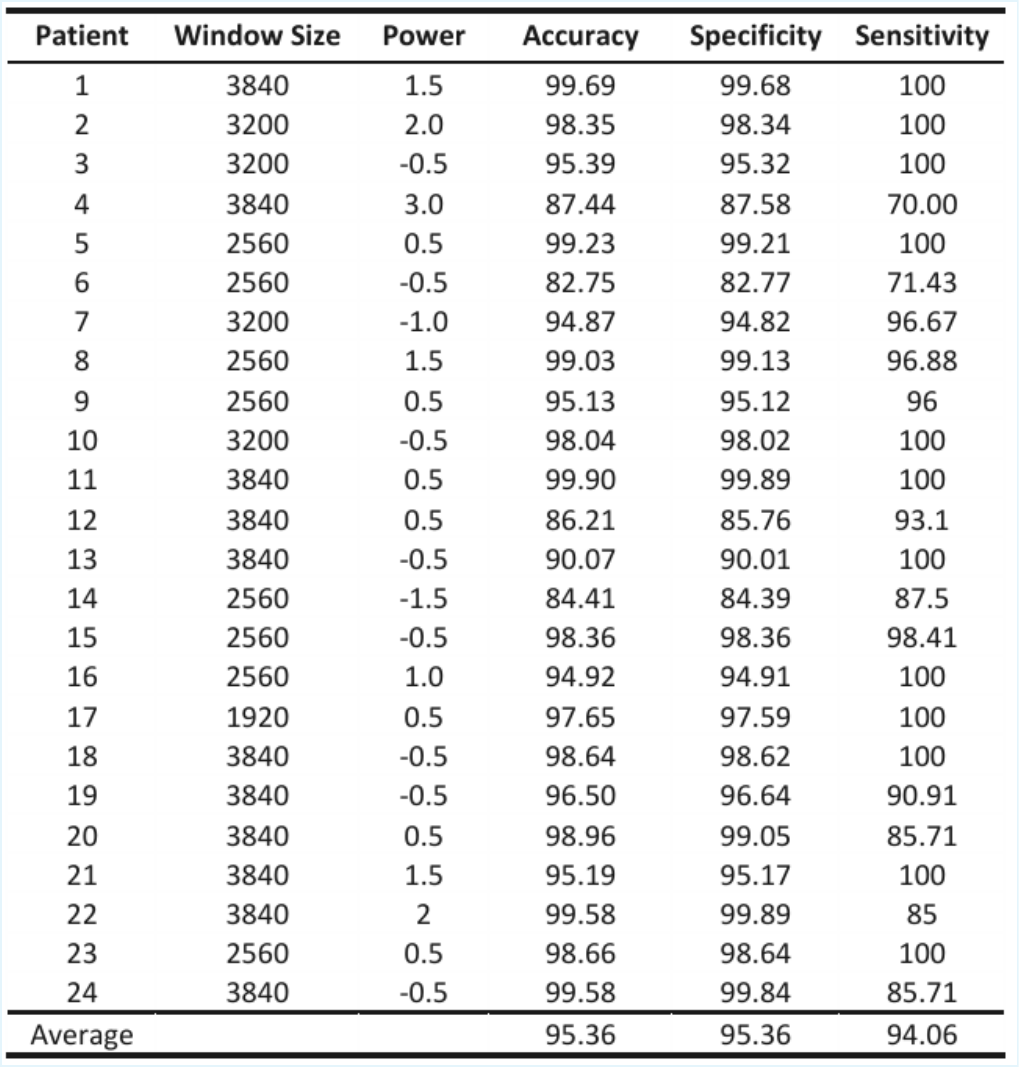}
	\caption{Performance for an LSTM trained with our features across distinct patients.}
	\label{fig:features_across_pat3}
\end{figure}

Finally, Fig.~\ref{fig:learned_exponents} presents the distribution of the learned feature representations (optimum power) across the patient cohort. The absence of a universal optimum exponent supports the hypothesis that individual brain activity is governed by unique dynamical laws, necessitating personalized structural estimators. Consequently, an adaptive approach—such as the one proposed—for selecting the optimum feature representation is critical for robust performance in this scenario.
\begin{figure}[t]
	\centering
	\includegraphics[scale=0.5]{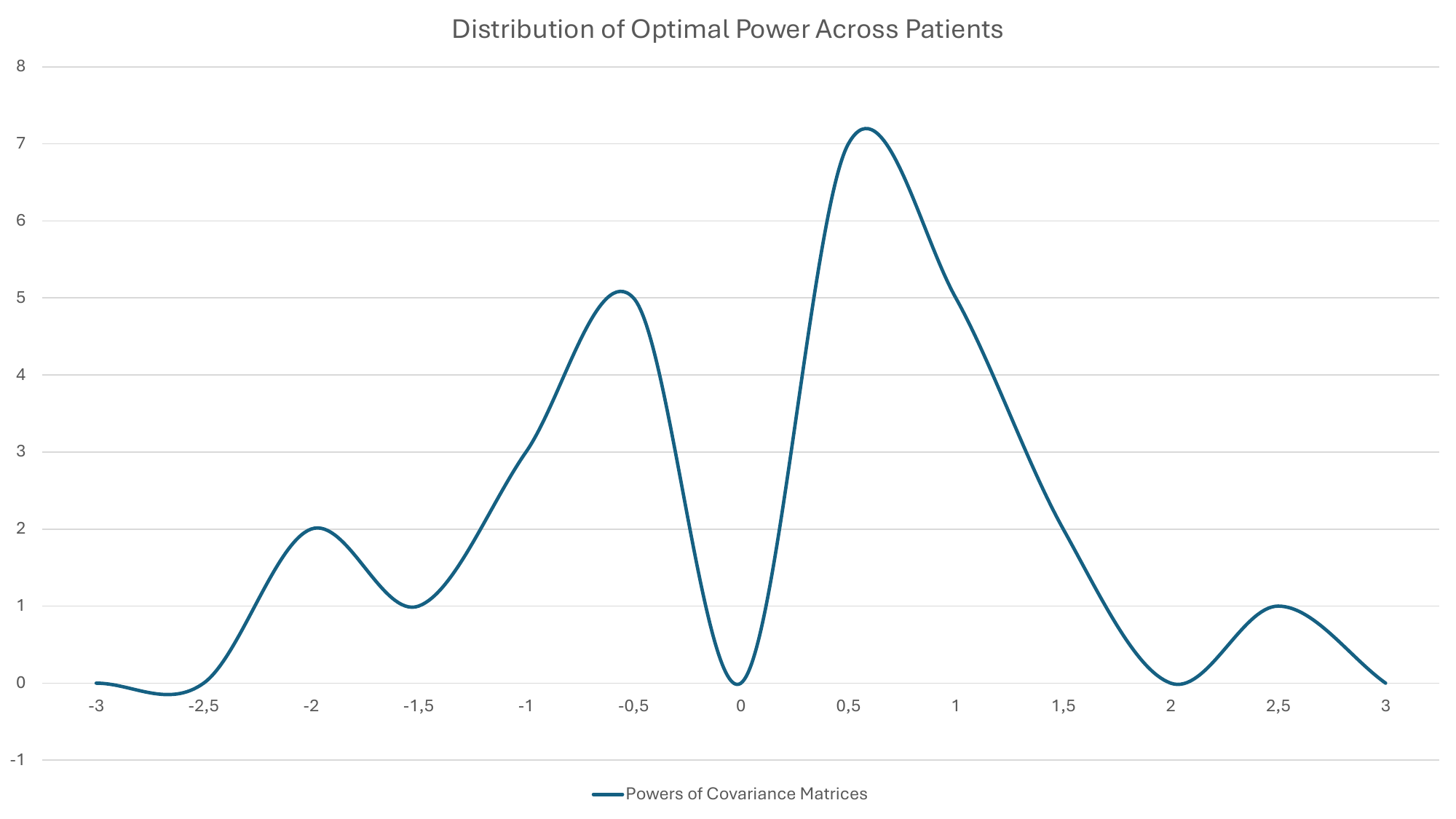}
	\caption{Distribution of the learned exponent across patients.}
	\label{fig:learned_exponents}
\end{figure}

\small
\clearpage
\bibliographystyle{IEEEtran}
\bibliography{IEEEabrv,biblio}


\end{document}